\documentclass{article}

\PassOptionsToPackage{numbers, compress}{natbib}
\usepackage[main, final]{neurips_2026}

\usepackage[utf8]{inputenc} 
\usepackage[T1]{fontenc}    
\usepackage{hyperref}       
\usepackage{url}            
\usepackage{booktabs}       
\usepackage{amsfonts}       
\usepackage{nicefrac}       
\usepackage{microtype}      
\usepackage{xcolor}         

\usepackage{algorithm}
\usepackage{algorithmic}
\usepackage{amsmath}
\usepackage{amssymb}
\usepackage{mathtools}
\usepackage{amsthm}
\usepackage{cleveref}
\usepackage{multirow}
\usepackage{pdfpages}
\usepackage{ulem}

\theoremstyle{plain}
\newtheorem{theorem}{Theorem}
\newtheorem{definition}{Definition}
\newtheorem{proposition}{Proposition}
\newtheorem{statement}{Statement}
\newtheorem{lemma}{Lemma}

\title{The Weight Gram Matrix Captures\\ Sequential Feature Linearization in Deep Networks}

\author{
  Taehun Cha \\
  Department of Mathematics \\
  Korea University \\
  \texttt{cth127@korea.ac.kr} \\
  \And
  Daniel Beaglehole \\
  Computer Science and Engineering \\
  UC San Diego \\
  \texttt{dbeaglehole@ucsd.edu} \\
  \And
  Adityanarayanan Radhakrishnan \\
  MIT Mathematics \\
  Broad Institute of MIT and Harvard \\
  \texttt{aradha@mit.edu} \\
  \And
  Donghun Lee\thanks{Corresponding author.} \\
  Department of Mathematics \\
  Korea University \\
  \texttt{holy@korea.ac.kr} \\
}

\begin{document}

\maketitle

\begin{abstract}
    Understanding how deep neural networks learn representations remains a central challenge in machine learning theory. 
    In this work, we propose a feature-centric framework for analyzing neural network training by relating weight updates to feature evolution. 
    We introduce a simple identity, the \textit{Feature Learning Equation}, which identifies the weight Gram matrix as the key object capturing feature dynamics. 
    This enables us to interpret gradient descent as implicitly inducing a hypothetical evolution of features, whose covariance structure --- termed the \textit{Virtual Covariance} --- characterizes how representations evolve during training.
    Building on this perspective, we introduce \textit{Target Linearity}, a measure quantifying the linear alignment between features and targets. 
    By analyzing the training and layer-wise dynamics, we show that deep networks learn to \textit{sequentially} transform representations toward target-linear structure. 
    This linearization perspective provides a unified interpretation of several empirical phenomena, including Neural Collapse and linear interpolation in generative models.
    Code is available at \url{https://github.com/cth127/GramLin}.
\end{abstract}

\section{Introduction}
\label{sec1:intro}

Feature learning is central to the success of deep neural networks, yet a unified account of how representations evolve during training remains limited. 
Existing work falls into two directions. 
General theories of training dynamics are largely parameter-centric, as in kernel-based analyses where features remain effectively fixed~\citep{jacot2018neural}. 
In contrast, recent work studies feature learning through analytically tractable case studies --- e.g., single and multi-index models, staircase, polynomials, and teacher-student setups --- showing that networks can provably learn task-relevant representations~\citep{abbe2022merged, ba2022high, damian2022neural, dandi2025computational}.
However, these results rely on specific problem structures, leaving open how feature learning behaves in general, realistic, high-dimensional settings. 

A natural way to address this gap is to analyze training directly in feature space.
As a thought experiment, suppose that instead of updating weights, we directly update hidden features by following the loss gradient.
Concretely, fix the input $x$ and define a virtual trajectory $\tilde{x}_t$ that accumulates the feature-space gradient at each training step:
\[
\tilde{x}_{t+1} \leftarrow \tilde{x}_t - \gamma \nabla_x \mathcal{L}_t, \qquad \tilde{x}_0 = x,
\]
where $\nabla_x \mathcal{L}_t$ is computed under the network with parameters $\theta_t$ with usual gradient descent update (GD) at step $t$.
We call $\tilde{x}_t$ the \textit{virtual update} of $x$.
Note that $\tilde{x}_t$ does not affect the actual training trajectory; it tracks, in feature space, the direction GD would move features if features were the optimization variables.
\Cref{fig:swissroll} shows that the virtual update progressively unrolls Swiss roll data into a linear form.
This raises a natural question: although GD operates in weight space, does it implicitly induce this feature-space evolution?
Answering it requires a tool that connects weight-space updates to feature-space dynamics.

Such a tool follows directly from the chain rule.
We begin by introducing a simple identity relating gradients in weight space to gradients in input (or hidden) space.
For any function of the form $f(z)$ with $z = Wh$, the chain rule yields
\[
\nabla_h f = W^\top \nabla_z f \quad \text{and} \quad \nabla_W f = \nabla_z f \cdot h^\top 
\]
Combining these gives
\begin{equation}
    \nabla_h f \cdot h^\top = W^\top \nabla_W f.
    \label{eq:fle}
\end{equation}
The left-hand side describes how features evolve; the right-hand side describes how weights update. The equation allows us to translate between the two, and in particular, it identifies the weight Gram matrix $W^\top W$ as the natural object connecting them.
We refer to this exact identity as the \textit{Feature Learning Equation}, and use it throughout the paper to analyze feature evolution.

In \Cref{sec2:gram}, we examine the role of the weight Gram matrix more closely and present empirical evidence that it carries the dominant signal of feature learning.
This enables us to analyze \textit{what the weight Gram matrix learns} and, in turn, \textit{how hidden features evolve} during training --- i.e., \textit{feature learning}.
Building on this perspective, we make four theoretical contributions:
\begin{itemize}
    \item \Cref{sec3:vc}: We prove that the weight Gram matrix learns the \textit{Virtual Covariance}, which captures the covariance structure of virtually updated hidden states.
    \item \Cref{sec4.sub1:tl}: We introduce \textit{Target Linearity}, a novel quantity measuring the linear relationship between hidden states and target labels, where the Gram structure appears again.
    \item \Cref{sec4.sub2:dynamics}: We characterize the training and layer-wise dynamics of Target Linearity. The training dynamics reveal that feature learning is a moving-target problem in which the network chases an evolving linear predictor, in contrast to the Neural Tangent Kernel regime (NTK, \citep{jacot2018neural}) where the linear model is fixed. The layer-wise dynamics show that trained networks \textit{sequentially} linearize input data with respect to the target, consistent with prior empirical findings on \textit{linear probes}~\citep{alain2016understanding}.
    \item \Cref{sec5:other}: We show this linearization perspective can explain other feature learning phenomena. We show Neural Collapse~\citep{papyan2020prevalence} is an extreme case of target linearization, and that high Target Linearity is a sufficient condition for the linear interpolation property observed in the latent space of generative models~\citep{shao2018riemannian}, demonstrating the generality of our framework.
\end{itemize}

\begin{figure}[t]
    \centering
    \includegraphics[width=0.98\linewidth]{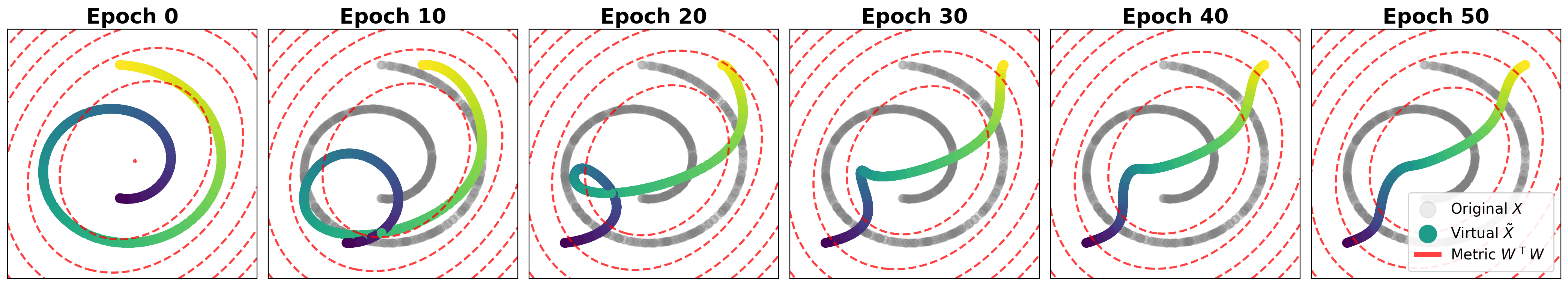}
    \caption{
    The virtual update $\tilde{X}_{t+1} \leftarrow \tilde{X}_t - \gamma \nabla_X \mathcal{L}_t$ progressively unrolls a Swiss roll into a near-linear curve along the target (encoded by color), shown across training epochs.
    Gray points denote the original input $X$; colored points denote the virtually updated $\tilde{X}$.
    The red dashed contours show the geometry induced by $W^\top W$ at the first layer --- a connection developed in \Cref{sec3:vc}.
    The detailed experimental setup is provided in \Cref{a2.sub0:swissroll}.
    }
    \label{fig:swissroll}
\end{figure}

\section{Weight Gram Matrix as a Bridge}
\label{sec2:gram}

We now show how \Cref{eq:fle} surfaces when analyzing the weight Gram matrix.
Computing the Gram matrix of the gradient-descent-updated weight matrix yields
\begin{equation}
    \label{eq:gram}
    \begin{aligned}
        (W - \gamma \nabla_W \mathcal{L})^\top (W - \gamma \nabla_W \mathcal{L}) 
        &= W^\top W - \gamma (W^\top \nabla_W \mathcal{L} + \nabla_W \mathcal{L}^\top W) + O(\gamma^2)\\
        &\approx W^\top W - \gamma (\nabla_h \mathcal{L} \cdot h^\top + h \cdot \nabla_h \mathcal{L}^\top).
    \end{aligned}
\end{equation}
where $\mathcal{L}$ is the loss function and $\gamma$ is the learning rate.
This identity makes a key point clear: \textit{by taking the Gram of the weight matrix, we can directly access feature-space gradients and, through them, feature evolution}.

The weight Gram matrix is not an entirely novel object in the feature learning literature.
\citet{radhakrishnan2024mechanism} recently proposed the \textit{Neural Feature Ansatz} (NFA) as a step toward understanding feature learning in neural networks.
They showed that the weight Gram matrix at each layer is closely related to the gradient of the network output with respect to that layer's input features.
For a network $f$ with $l$-th layer weight matrix $W_l$ and corresponding input representation $h_l$, NFA conjectures
\[
W_l^\top W_l \propto \mathbb{E}\left[ \nabla_{h_l} f \cdot \nabla_{h_l} f^\top \right],
\]
where the right-hand side is termed the \textit{Average Gradient Outer Product} (AGOP).
This relationship has been established theoretically in low-rank matrix completion, a widely studied setting for feature learning~\cite{radhakrishnan2025linear}, and has also been validated empirically across a broad range of architectures and datasets.

The discussion so far establishes the weight Gram matrix as a natural object in feature learning analysis, both theoretically (via the FL equation) and through prior empirical observation (via NFA).
A sharper question remains: Can the weight Gram matrix sufficiently capture the feature learning dynamics?
We test this directly by isolating the Gram-modifying component of the gradient update.

Specifically, we consider an update direction $\Delta'$ that leaves the weight Gram matrix unchanged, found by solving
\[
\arg\min_{\Delta'} \| \nabla_W \mathcal{L} - \Delta' \|_F \quad \text{subject to} \quad W^\top \Delta' + \Delta'^\top W = 0.
\]
Setting $\Delta =  \nabla_W \mathcal{L} - \Delta'$ and updating $W^+ = W - \gamma \Delta$ yields
\[
(W^+)^\top W^+ = (W - \gamma \nabla_W \mathcal{L})^\top (W - \gamma \nabla_W \mathcal{L}) + O(\gamma^2),
\]
so the Gram matrix evolves identically to that of standard GD, while the components of the gradient that leave the Gram unchanged are removed.
We refer to $\Delta$ as the \textit{whitened gradient}.
If the weight Gram matrix carries the learning signal, training with $\Delta$ should recover the performance of standard GD; if it does not, performance should largely degrade.

\begin{figure}
    \centering
    \includegraphics[width=0.98\linewidth]{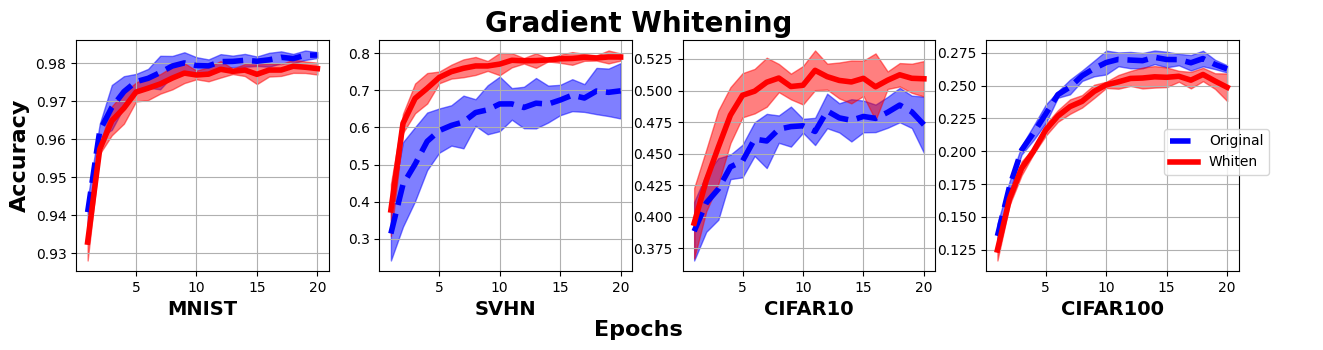}
    \caption{Test accuracy of standard gradient descent (\textcolor{blue}{Original}, dashed) versus the whitened update (\textcolor{red}{Whiten}, solid) on MNIST, SVHN, CIFAR-10, and CIFAR-100, plotted across 20 training epochs. Shaded regions denote standard deviation across 5 runs.}
    \label{fig:whitening}
\end{figure}

We train three-layer fully connected networks on standard image classification benchmarks \citep{krizhevsky2009learning, lecun2002gradient, netzer2011reading} using both standard GD and the whitened update.
\Cref{fig:whitening} reports the results.
On MNIST, SVHN, and CIFAR-10, the whitened update matches or exceeds the performance of standard GD; on SVHN and CIFAR-10 specifically, it reaches GD's best performance within four to five epochs and shows lower variance across runs on SVHN. 
On CIFAR-100, it tracks standard GD with a small gap.
The details for computing $\Delta$ are described in \Cref{a2.sub0:whiten}, and comparable results for convolutional architectures~\citep{krizhevsky2012imagenet} including ResNet~\citep{he2016deep} are deferred to \Cref{a2.sub1:whiten_cnn}.
Sensitivity to width, depth, and optimization hyperparameters is reported in \Cref{a2.sub3:whiten_sensitivity}.

These results indicate that the Gram-modifying component of the gradient is largely sufficient to recover standard training performance: the residual component, while present in the gradient, contributes little to feature learning beyond what the Gram update already provides.
This empirically substantiates the central role of the weight Gram matrix and motivates the analyses developed in the remainder of the paper, where we ask what the Gram matrix learns (\Cref{sec3:vc}) and what its evolution achieves at the level of features (\Cref{sec4:tl}).

\section{Weight Gram Learns Virtual Covariance}
\label{sec3:vc}

\Cref{sec2:gram} established the weight Gram matrix as a natural and empirically central object in feature learning analysis.
We now ask: \textit{what does the weight Gram matrix learn during training?}
We show that each GD step shifts the weight Gram matrix by an amount that captures the covariance structure of \textit{virtually updated} hidden states --- the same virtual update introduced in \Cref{sec1:intro}.
This formalizes the visualization in \Cref{fig:swissroll}: the Gram matrix tracks how features would evolve if features themselves were the optimization variables.

\paragraph{Settings} 

We work with a standard $L$-layer fully connected network $f(x) = w^\top h_L$, where $h_l \in \mathbb{R}^{d_l}$ denotes the $l$-th hidden state, defined recursively as $h_l = \sigma(W_{l-1} h_{l-1})$ with $h_0 = x \in \mathbb{R}^{d_0}$ and $\sigma$ a nonlinear activation function applied element-wise.
The vector $w \in \mathbb{R}^{d_L}$ serves as the final predictor, and $W_l \in \mathbb{R}^{d_{l+1} \times d_l}$ are the weight matrices.
For a loss function $\mathcal{L}$, we write $\nabla_{W_l} \mathcal{L} \in \mathbb{R}^{d_{l+1} \times d_l}$ for the gradient with respect to $W_l$.

\begin{definition}
The \textbf{virtual update} of a hidden state $h_l$ (or of the input $h_0 = x$) with learning rate $\gamma$ is $h^+_l = h_l - \gamma \nabla_{h_l} \mathcal{L}$.
The \textbf{actual update} of a weight matrix $W_l$ is $W^+_l = W_l - \gamma \nabla_{W_l} \mathcal{L}$, the standard GD step.
The \textbf{Virtual Covariance} (VC) is $\widetilde{\mathrm{cov}}(h^+_l) = h^+_l \cdot (h^+_l)^\top$, and the \textbf{Virtual Covariance Shift} (VCS) is $\widetilde{\mathrm{cov}}(h^+_l) - \widetilde{\mathrm{cov}}(h_l)$.
\end{definition}

The virtual update represents how the input to the $l$-th layer would change under GD if hidden states themselves were the optimization variables.
We use the term \textit{virtual} because GD updates weights, not hidden states directly.
The VCS therefore characterizes how the input covariance structure would evolve to reduce the loss.\footnote{Strictly, $\widetilde{\mathrm{cov}}$ is not a covariance matrix since $h_l^+$ and $h_l$ are not guaranteed to be centered; with standard normalization the two coincide.}

\paragraph{Weight Gram Dynamics}

The following theorem, our main result for this section, shows that the actual GD update on the weights produces precisely this VCS on the Gram matrix.

\begin{theorem}
    \label{thm:vcs}
    The shift in the Gram matrix of the gradient-descent-updated weight matrix equals the Virtual Covariance Shift up to $O(\gamma^2)$:
    \begin{equation*}
        (W^+_l)^\top W^+_l - W_l^\top W_l = \widetilde{\mathrm{cov}}(h^+_l) - \widetilde{\mathrm{cov}}(h_l) + O(\gamma^2).\footnote{For a mini-batch update $W^+_l = W_l - \gamma \sum_i \nabla_{W_l} \mathcal{L}_i$, the Gram shift accumulates the per-sample VCS, $(W^+_l)^\top W^+_l - W_l^\top W_l \approx \sum_i [\widetilde{\mathrm{cov}}(h^+_{l,i}) - \widetilde{\mathrm{cov}}(h_{l,i})]$. Applying over $\tau$ consecutive steps, the Gram shift extends directly to the multi-step virtual update, $W_t^\top W_t - W_{t-\tau}^\top W_{t-\tau} \approx \tilde{h}_t \tilde{h}_t^\top - \tilde{h}_{t-\tau} \tilde{h}_{t-\tau}^\top$, where $\tilde{h}_t$ denotes the virtual trajectory defined in \Cref{sec1:intro}.}
    \end{equation*}
\end{theorem}

The proof is direct from \Cref{eq:gram}, as $\widetilde{\mathrm{cov}}(h^+_l) - \widetilde{\mathrm{cov}}(h_l) = - \gamma (\nabla_h \mathcal{L} \cdot h^\top + h \cdot \nabla_h \mathcal{L}^\top) + O(\gamma^2)$.
We show this also holds for Convolutional neural network~\citep{krizhevsky2012imagenet} in \Cref{a2.sub0:vcs_cnn}.

The $i$-th column $w_i$ of $W$ determines how the $i$-th input entry $x_i$ influences the next layer $h = \sigma(Wx)$.
The Gram entry $(W^\top W)_{ij} = w_i^\top w_j$ therefore measures the correlation between the effects of $x_i$ and $x_j$ on $h$.
In this sense, the weight Gram matrix is an \textit{input covariance with respect to its influence on the next-layer representation}, while the VCS describes how this covariance \textit{should} evolve to reduce the loss.
\Cref{thm:vcs} shows that GD on the weights aligns the two: the actual Gram update tracks the desired covariance shift, even though hidden states are never updated directly.
The red dashed contours in \Cref{fig:swissroll} make this evolution visible: each contour shows the density of a Gaussian with covariance $W^\top W$ at the first layer, adapting with the virtual unrolling of the spiral.

This result clarifies the relationship to the Neural Feature Ansatz of \citet{radhakrishnan2024mechanism}, which conjectures $W^\top W \propto \tfrac{1}{N} \sum_{i=1}^N \nabla_{h_i} f \cdot \nabla_{h_i} f^\top$ (AGOP).
An AGOP-like term --- with the network output replaced by the loss --- appears within the $\gamma^2$ residual of \Cref{thm:vcs}.
The VCS therefore refines the AGOP correspondence by capturing the leading-order Gram update directly.
Empirically, the VCS tracks the actual Gram update at $\rho = 0.98$, compared to $\rho = 0.71$ for AGOP (See \Cref{a1.sub2:vsagop}).

\paragraph{Relationship between Virtual vs. Actual Update}

\Cref{thm:vcs} establishes that the weight Gram matrix learns the covariance structure of \textit{virtually} updated features.
A natural question is whether this is meaningful in practice: does the \textit{actually} updated hidden state $\sigma(W^+_{l-1} h_{l-1})$ behave similarly to the virtual update $h_l^+$?
If actual and virtual updates diverge, $W_l$ may be optimized to expect features that the previous layer never produces.
The following proposition, stated for the single-sample case, shows the two are aligned in both sign and magnitude.

\begin{proposition}
\label{prop:vc_alignment}
Assume $\sigma$ is $L$-Lipschitz and $\|h_{l-1}\|_2 = 1$, and let $W^+_{l-1}$ denote the layer $l{-}1$ weight matrix after one gradient descent step on the single data point producing $h_{l-1}$, with all other layers held fixed.
Then $|\sigma(W^+_{l-1} h_{l-1}) - \sigma(W_{l-1} h_{l-1})| \leq L^2 |h_l^+ - h_l|$ element-wise.
If $\sigma$ is additionally monotonically increasing, then $\mathrm{sgn}(\sigma(W^+_{l-1} h_{l-1}) - \sigma(W_{l-1} h_{l-1})) = \mathrm{sgn}(h_l^+ - h_l)$; the sign equality also holds for non-decreasing $\sigma$ such as ReLU under the convention $\mathrm{sgn}(0) \in \{-1, +1\}$.
\end{proposition}

\Cref{prop:vc_alignment} shows that the actual update on the hidden state agrees with the virtual update in direction and is bounded by it in magnitude.
For 1-Lipschitz activations such as ReLU, the magnitude bound reduces to $|h_l^+ - h_l|$ directly.
The proof is given in \Cref{a2.sub1:vc_alignment}.

However, \Cref{prop:vc_alignment} assumes that only a single layer is updated.
When all layers are trained simultaneously, the actual update can still be expressed as a preconditioned virtual update under gradient flow.

\begin{proposition}
    \label{prop:all_layer}
    Consider a single input $x$ and an $L$-layer network with $z_l = W_{l-1} h_{l-1}$ and $h_l = \sigma(z_l)$ for $l = 1, \dots, L$, where $h_0 = x$ and $\sigma$ is differentiable.
    Suppose all weights evolve under gradient flow, $\dot{W}_{l} = -\nabla_{W_{l}} \mathcal{L}$ for $l = 0, \dots, L-1$.
    Then for every layer $l$,
    \begin{equation*}
        \frac{d h_l}{dt} = -M_l \nabla_{h_l} \mathcal{L}, \qquad M_l = \sum_{l'=1}^{l} \|h_{l'-1}\|_2^2 \, J_{l \leftarrow l'} J_{l \leftarrow l'}^\top \succeq 0,
    \end{equation*}
    where $J_{l \leftarrow l'} = \partial h_l / \partial z_{l'}$.
    Consequently, $\langle \frac{d h_l}{dt}, -\nabla_{h_l} \mathcal{L} \rangle \ge 0$.
\end{proposition}

\Cref{prop:all_layer} shows that, even when all layers are trained, the actual feature update is the virtual update $-\nabla_{h_l}\mathcal{L}$ preconditioned by a PSD matrix, and thus never moves against it.
The proof is given in \Cref{a2.sub2:proof_all_layer}.

We established that the weight Gram matrix tracks the Virtual Covariance Shift --- the covariance structure that GD implicitly imposes on features. 
This tells us how features move, but not where they are going. \Cref{fig:swissroll} already hinted at an answer: the virtual update doesn't just reshape covariance arbitrarily, it progressively aligns features with the target. 
To make this observation precise, we need a quantity that measures alignment between features and targets, expressed in terms of the same Gram structure that \Cref{thm:vcs} identified as the carrier of feature dynamics. 
This motivates \textit{Target Linearity}, introduced next: by construction, it is a function of the weight Gram matrix, so the covariance dynamics translate directly into dynamics of Target Linearity, letting us track not just how representations evolve, but how that evolution serves the task.

\section{Weight Gram Captures Target Linearity}
\label{sec4:tl}

We now formalize the linearization phenomenon hinted in \Cref{sec3:vc}.
We introduce \textit{Target Linearity}, a quantitative measure of how well a linear model can predict the target from learned features, and analyze its dynamics during training and across layers.
A central technical step is to identify a tractable surrogate that exposes the weight Gram matrix as the object that captures this linearization.

\subsection{Target Linearity and its Surrogate}
\label{sec4.sub1:tl}

\begin{definition}
For a data matrix $H = [h_{1}, \dots, h_{N}] \in \mathbb{R}^{d_l \times N}$ and target vector $Y = [y_1, \dots, y_N]^\top \in \mathbb{R}^N$, we define the \textbf{Target Linearity} (TL) as the coefficient of determination ($R^2$) of (ridge) linear regression: for $\lambda \geq 0$,
\[
\mathcal{T}_\lambda(H, Y) = 1 - \frac{\|Y - \hat{Y}_\lambda\|_2^2}{\|Y - \bar{Y}\|_2^2},
\]
where $\hat{Y}_\lambda = H^\top (HH^\top + \lambda I)^{-1} H Y$, $\bar{Y} = \bar{y} \cdot \mathbf{1}_N$, and $\bar{y} = \frac{1}{N} \sum_{i=1}^N y_i$.
\end{definition}

At $\lambda = 0$ (and assuming $HH^\top$ is invertible), TL recovers the $R^2$ of ordinary least squares.
Since TL is defined as the $R^2$ score of linear regression, it quantifies how well the target can be explained by a linear function of the features.

The features $H$ appear in the squared error $\|Y - \hat{Y}_\lambda\|_2^2$ through multiple occurrences, making the dependence difficult to analyze directly.
Applying the Woodbury identity yields $\|Y - \hat{Y}_\lambda\|_2^2 = \lambda^2 \, Y^\top (\lambda I + H^\top H)^{-2} Y$, which collapses the dependence on $H$ into a single Gram-like term $H^\top H$.
We define the Gram-induced error function $\mathcal{E}_\lambda(G^l_\sigma) = Y^\top (\lambda I + G^l_\sigma)^{-2} Y$, where $G^l_\sigma = \sigma(W_{l} H_{l})^\top \sigma(W_{l} H_{l})$.
We drop the superscript $l$ when the layer is clear from context.

However, $\mathcal{T}_\lambda$ still involves a squared matrix inverse and the nonlinearity $\sigma$ inside $G_\sigma$, both of which obstruct direct analysis.
To make the dependence on the Gram matrix tractable, we introduce a \textit{surrogate} $\mathcal{S}(G) = Y^\top G Y$.
Here $G$ may denote either the nonlinear Gram $G_\sigma = \sigma(WH)^\top \sigma(WH)$ or the \textit{linearized} Gram $G_{\mathrm{id}} = H^\top W^\top W H$, in which the activation has been removed.
The surrogate admits the following lower-bound relation to TL.

\begin{theorem}
    \label{thm:surrogate}
    Let $e_0 = \|Y - \bar{Y}\|_2^2$ and $e_1 = \|Y\|^2_2$, and assume $\|W\|_{\mathrm{op}} \leq c_0$, $\|H\|_F \leq c_1$, and $N > d$.
    Then
    \[
    \mathcal{T}_\lambda(H, Y) \geq 1 - C \cdot \mathcal{S}(G_{\mathrm{id}})^{-1},
    \]
    for some fixed constant $C$.
\end{theorem}

\Cref{thm:surrogate} yields two benefits.
First, the surrogate removes both the nonlinearity and the squared matrix inverse that obstruct direct analysis of $\mathcal{T}_\lambda$.
Second, $\mathcal{S}(G_{\mathrm{id}}) = Y^\top H^\top W^\top W H Y$ exposes the weight Gram matrix $W^\top W$ explicitly, connecting the analysis of TL to the results of \Cref{sec3:vc}.
The proof is given in \Cref{a2.sub2:surrogate}.
We now use this surrogate to study how feature evolves during training and across layers.

\subsection{Training \& Layer-wise Dynamics}
\label{sec4.sub2:dynamics}

\paragraph{Training Dynamics}

\begin{figure}
    \centering
    \includegraphics[width=0.98\linewidth]{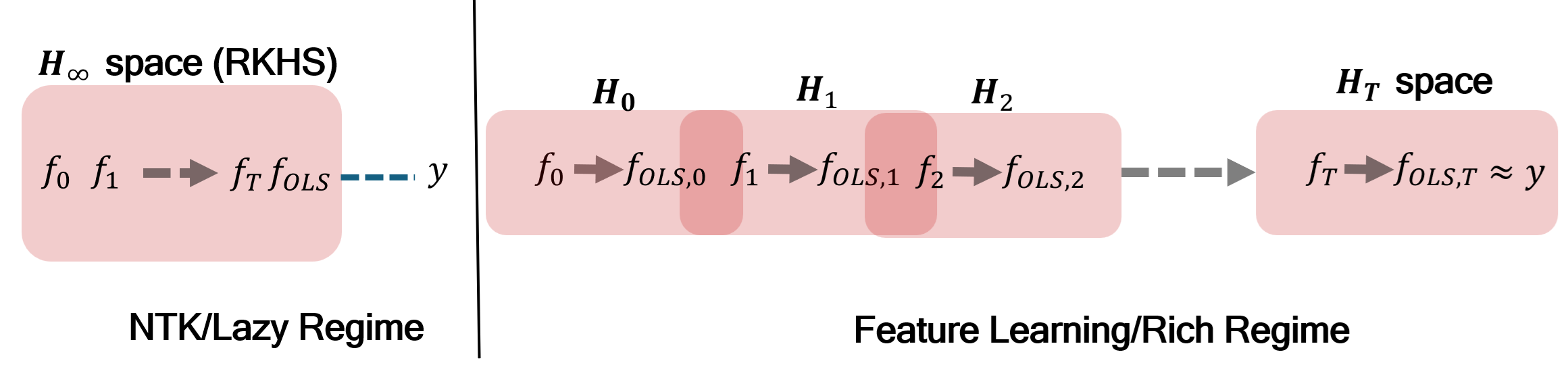}
    \caption{
    Training dynamics in the lazy (left) versus rich regimes (right).
    In the lazy regime, the network remains close to its linearization, and the feature map $H_\infty$ is effectively fixed.
    In the rich regime, finite-width features $H_t$ evolve during training, inducing a moving target $f_{\mathrm{OLS},t}$ that $f_t$ chases, while feature updates simultaneously reshape $f_{\mathrm{OLS},t+1}$ toward $y$.
    }
    \label{fig:dynamics}
\end{figure}

We now analyze how $\mathcal{S}(G_{\mathrm{id}})$ evolves under GD.
Throughout this subsection, we largely depend on a first-order Taylor expansion to characterize the leading-order behavior and see what this approach reveals about dominant portion of feature dynamics.
We assume $f$ is $m$-positively homogeneous in $h$: $f(\alpha h) = \alpha^m f(h)$ for any positive scalar $\alpha$.
We further assess the validity of this approach empirically in settings that go beyond these assumptions, including GELU activations in \Cref{sec4.sub3:empirical}; further discussion and theoretical justification of our approach are in \Cref{a1.sub2:limit}.

\begin{theorem}
\label{thm:training}
For any loss function $\mathcal{L}$, let $G_{\mathrm{id}} = H^\top W^\top W H$ and $G_{\mathrm{id}}^+ = H^\top (W^+)^\top W^+ H$, where $W^+ = W - \gamma \nabla_W \mathcal{L}$.
If $f$ is $1$-positively homogeneous in $h$, the following holds:
\[
\mathcal{S}(G_{\mathrm{id}}^+) - \mathcal{S}(G_{\mathrm{id}}) \approx 2 \gamma \, (f^\top y) \cdot (y^\top K g),
\]\footnote{Up to this point, $\approx$ denotes equality up to $O(\gamma^2)$ terms. From here on, it may additionally absorb the error of first-order Taylor approximations (\Cref{lemma:homogeneous}) and of the Neural Feature Ansatz (\Cref{thm:depth}), so $\approx$ in the theorem statements can combine several approximations. We refer to the corresponding proofs for which approximation is used at each step.}
where $K_{ij} = h_i^\top h_j$, $f_i = f(h_i)$ are the predictions on the training set, and $g = -\nabla_f \mathcal{L}$.\footnote{The same result extends to general $m$-positively homogeneous case under the centered label assumption, $\sum_i y_i = 0$.}
\end{theorem}

\begin{proof}[Proof Sketch]
We rewrite $(W^+)^\top W^+ - W^\top W$ using the FL equation (\Cref{eq:fle}), decompose $\nabla_h \mathcal{L} = -\nabla_h f \cdot g$, and resolve the cross terms $h_i^\top \nabla_{h_k} f$ via the $m$-homogeneity identity $h_i^\top \nabla_{h_k} f \approx f(h_i) + (m-1) f(h_k)$ (\Cref{lemma:homogeneous}).
The full proof is in \Cref{a2.sub3:training}.
\end{proof}

The term $f^\top y$ captures the empirical correlation between predictions and labels.
At initialization it is near zero, and as training proceeds it grows positive as $f$ aligns with $y$.

To build intuition for $y^\top K g$, consider the binary classification setting.
Under the binary cross-entropy loss with sigmoid predictions, $\hat{y}=\text{sigmoid}(f(x)) \in (0,1)^n$, and labels $y \in \{0,1\}^n$, the negative gradient takes the form $g = y - \hat{y}$, so $g_i > 0$ when $y_i = 1$ and $g_i < 0$ when $y_i = 0$.
Writing the quantity in feature space yields
\begin{align*}
    y^\top K g = (Hy)^\top (Hg) = \langle \sum_{i \in S_1} h_i, \; \sum_{j \in S_1}(1-\hat{y}_j)\, h_j - \sum_{k \in S_0} \hat{y}_k\, h_k \rangle,
\end{align*}
where $S_a = \{i : y_i = a\}$.
$y^\top K g$ measures the correlation, in feature space, between the positive-class centroid and a residual direction that pulls toward positive examples while pushing away from negative examples.
Under feature learning regime where representations cluster by class, as $\langle h_i, h_j\rangle \geq \langle h_i, h_k\rangle$ holds for $i,j \in S_1, k \in S_0$.
This makes the inner product positive whenever within-class similarity dominates cross-class similarity, before it tapers to zero as $\hat{y} \rightarrow y$.

A sharper interpretation of this gap comes from relating $y$ to a classical statistical object.
Let $f_{\mathrm{OLS}} = H^\top (HH^\top)^{-1} H y$ be the best linear predictor from current features.
A direct calculation gives $\|y - f_{\mathrm{OLS}}\|_K = 0$, so $y$ and $f_{\mathrm{OLS}}$ coincide under the $K$-norm.
As training proceeds, $f$ moves toward $y$ and as a result $f \rightarrow f_{\mathrm{OLS}}$ in $K$-norm, while $f_{\mathrm{OLS}}$ itself moves toward $y$ as TL increases.

This identifies feature learning as a \textit{moving-target problem}.
The network $f_t$ chases $f_{\mathrm{OLS},t}$, but $f_{\mathrm{OLS},t+1}$ itself is reshaped as features $H_{t+1}$ evolve: each gradient step closes the residual $\|f_{\mathrm{OLS},t} - f_t\|_{K_t}$ while simultaneously reopening it through feature updates that shift $f_{\mathrm{OLS},t+1}$ toward $y$.
This contrasts sharply with the NTK regime~\citep{jacot2018neural}, where, at infinite width, the kernel remains effectively fixed during training, and learning reduces to kernel regression in the corresponding RKHS~\citep{rahimi2007random}. 
In this regime, the solution is determined by the initial kernel, and training corresponds to fitting $y$ within this static function space, without additional parametrization such as $\mu P$~\citep{yang2021tuning}.
The lazy regime corresponds to this linearized setting in which the feature map remains effectively fixed, whereas in the feature learning regime the evolving representation $H_t$ induces a sequence of changing targets $f_{\mathrm{OLS},t}$, so that optimizing $f_t$ simultaneously reshapes the feature geometry itself (\Cref{fig:dynamics}). 
We empirically compare these two regimes in \Cref{a1.sub2:lazy}.

\paragraph{Layer-wise Dynamics}

The role of depth in neural networks has been empirically validated in many practical architectures~\citep{he2016deep, vaswani2017attention}, yet a unified theoretical account remains limited.
Most existing work studies depth from an approximation perspective, showing that function composition is fundamentally more expressive than shallow expansion~\citep{telgarsky2016benefits}.
This line of work asks what functions are \textit{learnable} with depth, rather than how a \textit{learned} deep network processes input across its layers.

For the derivation we assume the Neural Feature Ansatz~\citep{radhakrishnan2024mechanism} holds.
\citet{boix-adsera2026fact} discuss why the NFA usually holds and the degenerate cases under which it does not.
Under the NFA, we show that an interpolating (i.e., sufficiently trained) network \textit{sequentially linearizes} input data toward the target at the final linear readout.

\begin{theorem}
\label{thm:depth}
    Let $f_l: h_l \mapsto f(x)$ denote the subnetwork mapping the $l$-th hidden state to the final output, and assume the activation $\sigma$ is $m$-positively homogeneous.
    Assume the network interpolates the training data and $\exists C_0 > 0 \text{ s.t. } \|W_l H_l\|^2_F \geq C_0, \forall l$.
    Under the Neural Feature Ansatz $W_l^\top W_l \approx \tfrac{C_l}{N} \sum_i \nabla_{h_{l,i}} f_l \, \nabla_{h_{l,i}} f_l^\top$ for layer-dependent constants $C_l > 0$, the surrogate satisfies
    \[
    \mathcal{S}(G^l_{\mathrm{id}}) \begin{cases}
    \gtrsim C_0 \|Y\|_2^2 \, \Phi^{-1}_l, & \text{if } m > 1 \text{ and } \sum_i y_i = 0, \\
    \gtrsim C_0 \|Y\|_2^4 \, \Psi^{-1}_l, & \text{if } m = 1 \text{ and } f_l \text{ is piecewise linear,}
    \end{cases}
    \]
    where $\Phi^{-1}_l$ and $\Psi^{-1}_l$ are monotonically increasing in $l$.
\end{theorem}

\begin{proof}[Proof Sketch]
    Using the NFA, we decompose the surrogate into a sum of terms of the form $h_{l,i}^\top \nabla_{h_{l,j}} f$.
    A first-order Taylor expansion and homogeneity then gives $\mathcal{S}(G^l_{\mathrm{id}}) \approx C_l \|Y\|_2^4$ for both cases.
    Similarly, the NFA lets us decompose $\|W_l H_l\|_F^2 = \mathrm{tr}(H_l^\top W_l^\top W_l H_l) \geq C_0$.
    Relating the two expressions shows that the lower bound of $C_l$ is monotonically increasing in $l$.
    The full proof is in \Cref{a2.sub4:depth}.
\end{proof}

The two cases of \Cref{thm:depth} differ in how the layer-wise complexity is measured.
In the smooth $m$-homogeneous case ($m > 1$), $\Phi_l$ is governed by the homogeneity gap $ \propto m^{L-l}$.
In the piecewise-linear case ($m = 1$), the lower bound $\Psi_l$ is governed by the number of linear regions of $f_l$~\citep{montufar2014number}.
Both quantities measure the same thing from different angles: how far the subnetwork $f_l$ remains from being a linear function of its input.
An interpolating network must produce $y$ from $x$ end-to-end, and $\Phi_l$ (or $\Psi_l$) measures the nonlinearity $f_l$ still has to supply to close the gap between $h_l$ and $y$.
At shallow layers, $f_l$ does most of the work and the complexity is large; at deep layers, $f_l$ approaches a linear map and the complexity shrinks.
Depth is the mechanism by which the total required nonlinearity is distributed across layers, and the surrogate $\mathcal{S}(G^l_{\mathrm{id}})$ --- inversely proportional to $\Phi_l$ or $\Psi_l$ --- correspondingly grows with depth.

This formalizes \textit{sequential linearization}: as $l$ increases, the residual nonlinearity $f_l$ must supply shrinks, and the network leaves a progressively more linear problem for subsequent layers.
The picture matches the standard empirical observation underlying linear probes~\citep{alain2016understanding}: deeper layers admit higher-accuracy linear classifiers, because by the time data reaches them, much of the nonlinear work has already been done.

\subsection{Empirical Validation}
\label{sec4.sub3:empirical}

\begin{figure}
    \centering
    \includegraphics[width=0.49\linewidth]{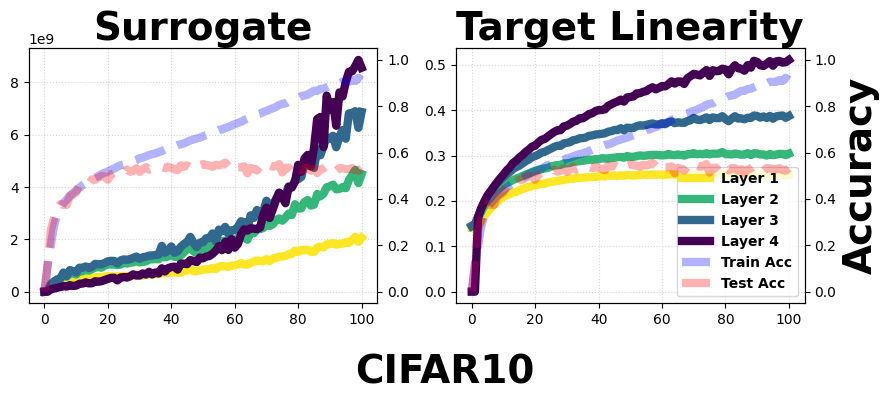}
    \includegraphics[width=0.49\linewidth]{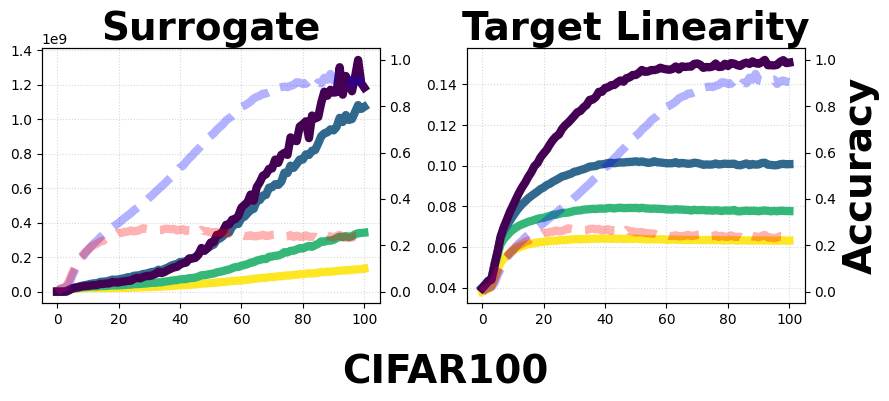}
    \caption{
    Training and layer-wise dynamics of the surrogate and Target Linearity for a 4-layer fully connected network trained on CIFAR-10 (left) and CIFAR-100 (right).
    Solid lines show layers 1--4 (yellow to dark purple); dashed lines show train and test accuracy.
    }
    \label{fig:empirical}
\end{figure}

We train a four-layer fully connected network with hidden dimension 256 and GELU activations~\citep{hendrycks2016gaussian} on CIFAR-10 and CIFAR-100~\citep{krizhevsky2009learning}, using SGD with learning rate $0.005$ for 100 epochs.
To extend the Target Linearity to the multiple target cases, we extend the surrogate to $C$-class classification by replacing $Y$ with the one-hot label matrix $Y \in \mathbb{R}^{N \times C}$ and redefining the surrogate as $\mathcal{S}_{\mathrm{cls}}(G) = \mathrm{tr}(Y^\top G Y)$.
Since $\mathrm{tr}(Y^\top G Y) = \sum_c y_c^\top G y_c$, the dynamics from \Cref{sec4.sub2:dynamics} carry over directly.

\Cref{fig:empirical} reports the training and layer-wise dynamics of the surrogate $\mathcal{S}_{\mathrm{cls}}(G^l_{\mathrm{id}})$ and TL ($\lambda=0$) at each layer.
Both quantities grow throughout training, consistent with the regime identified in \Cref{thm:training}, and both increase with layer index, consistent with the sequential linearization predicted by \Cref{thm:depth}.
The surrogate grows roughly exponentially while TL grows roughly logarithmically; this is consistent with \Cref{thm:surrogate}, where the lower bound on TL involves the inverse of the surrogate.
Analogous results for the Adam optimizer~\citep{kingma2014adam} are in \Cref{a1.sub3:adam}.

A notable benefit of TL is that, by being defined as the $R^2$ of regression onto the target, it applies to any task with a continuous target --- not just classification.
We confirm this on self-supervised generation: TL grows during VAE training~\citep{kingma2013auto} when the target is the reconstructed input (\Cref{a1.sub4:generative}).
This positions TL as a unifying measure of feature learning across discriminative and generative tasks, in contrast to phenomena like Neural Collapse~\citep{papyan2020prevalence} that are formulated specifically for classification.

\section{Link to Other Feature Learning Phenomena}
\label{sec5:other}

The framework developed in \Cref{sec4:tl} is general --- TL is defined for any task with a target, and its dynamics are governed by the weight Gram matrix.
We now show that this generality recovers two well-known feature learning phenomena as special cases of our framework: Neural Collapse~\citep{papyan2020prevalence} as the maximum case of our surrogate in classification, and the linear interpolation property of generative models~\citep{shao2018riemannian} as a consequence of high TL in latent space.

\paragraph{Neural Collapse as an Extreme Case}

In the classification setting, \citet{papyan2020prevalence} report the \textit{Neural Collapse} (NC) phenomenon, in which the penultimate-layer features of trained networks (NC1) collapse to their class means with vanishing within-class variation, and (NC2) form a simplex equiangular tight frame.
Existing theoretical accounts of NC rely largely on the Unconstrained Feature Model~\citep{mixon2022neural}, which treats hidden states as freely trainable variables; recent work has begun to relax this~\citep{min2025neural} but only under separability assumptions.

\begin{proposition}
\label{prop:nc_extreme}
Consider $f(x) = W^\top h(x) + b$ with hidden states centered ($\sum_i h(x_i) = 0$), $W$ and $b$ trained under squared loss on the centered one-hot label matrix, and $C \leq D$ for the last-layer hidden dimension $D$.
Under NC1 and NC2, $\mathcal{S}_{\mathrm{cls}}(G_{\mathrm{id}})$ attains the maximum value allowed by a norm constraint $\|G_{\mathrm{id}}\|_F \leq c$ for some $c > 0$.
\end{proposition}

This maximum is unique at the level of the Gram matrix: by the equality condition of the Cauchy--Schwarz inequality in \Cref{lemma:surrogate_bound}, $\mathcal{S}_{\mathrm{cls}}(G_{\mathrm{id}})$ is maximized if and only if $G_{\mathrm{id}} = c\, Y Y^\top / \|Y Y^\top\|_F$, which is well defined since the centered one-hot label matrix satisfies $Y Y^\top \neq 0$ for $C \ge 2$.

\Cref{prop:nc_extreme} shows that NC corresponds to a maximally linearized representation in classification: the geometric structure observed empirically in trained networks~\citep{papyan2020prevalence} is precisely what saturates the surrogate $\mathcal{S}_{\mathrm{cls}}$.
This recovers NC without invoking the Unconstrained Feature Model, instead deriving it from the same Gram-matrix framework that governs general feature learning.
The full proof is in \Cref{a2.sub5:nc_extreme}.
We note that~\citet{beaglehole2024average} independently show that target linearity (in the form of kernel linearization) controls the rate of NC formation in deep Recursive Feature Machines.

\paragraph{Linear Interpolation Property in Generative Models}

\begin{figure}
    \centering
    \includegraphics[width=0.9\linewidth]{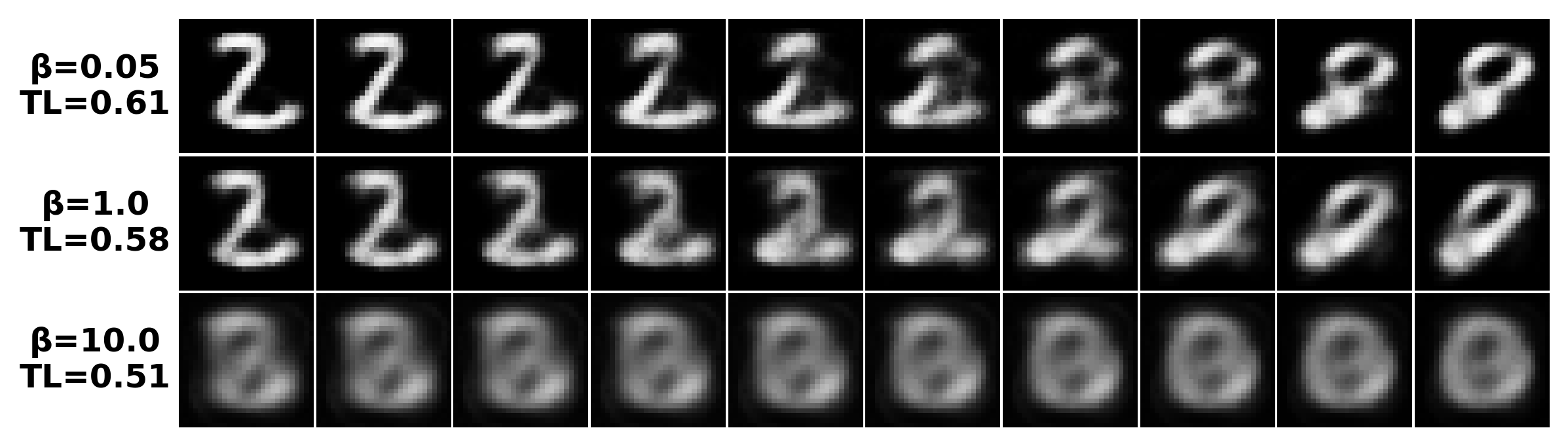}
    \caption{
    Latent-space interpolation between digits $2$ and $8$ in a VAE for different $\beta$ values (the KL divergence term parameter). Lower $\beta$ (higher TL) yields near-linear transitions in pixel space, while higher $\beta$ leads to blurrier and less faithful interpolations.
    }
    \label{fig:vae_interpolation}
\end{figure}

TL is defined as the $R^2$ of regression, so it applies to any task with a continuous target.
We use this to address a phenomenon in generative models: \citet{shao2018riemannian} observed that decoders of trained Variational Autoencoders~\citep{kingma2013auto} approximately respect linear interpolation in latent space, $f(\alpha z_1 + (1-\alpha) z_2) \approx \alpha x_1 + (1-\alpha) x_2$, where $f$ is the decoder, $z_i$ are latent codes, and $x_i$ are the corresponding inputs.
This raises a natural question: when does a nonlinear decoder admit such linear interpolation?

\begin{proposition}
\label{prop:vae_linearity}
Let $Z = [z_1, \ldots, z_N]^\top$ be latent vectors and $X = [x_1, \ldots, x_N]^\top$ the corresponding inputs.
Let $g(z) = z^\top (Z^\top Z)^{-1} Z^\top X$ be the ordinary least-squares estimator, and assume $g \in \mathcal{F}$.
Then for any $\alpha \in [0, 1]$,
\[
\min_{f \in \mathcal{F}} \tfrac{1}{N^2} \sum_{i,j} \|f(\alpha z_i + (1-\alpha) z_j) - (\alpha x_i + (1-\alpha) x_j)\|_2 \;\leq\; \tfrac{1}{N} \sum_i \|g(z_i) - x_i\|_2.
\]
\end{proposition}

The right-hand side is exactly the OLS reconstruction error, which equals square root of $1 - \mathcal{T}_0(Z, X)$ up to normalization.
\Cref{prop:vae_linearity} therefore states that whenever the function class contains the OLS predictor and the latent code admits high TL with respect to the inputs, the decoder can interpolate linearly.
\Cref{fig:vae_interpolation} empirically validates this finding in MNIST trained VAE (from class 2 to 8) and the full proof is in \Cref{a2.sub6:vae_linearity}.

\section{Related Work \& Discussion}
\label{sec6:discussion}

\paragraph{Related Work}

Feature linearization has been observed empirically across several settings: \cite{brahma2015deep} reported decreasing curvature of hidden representations with depth, and \cite{alain2016understanding} showed linear probes attain higher accuracy at deeper layers.
These observations are diagnostic; TL formalizes them as a quantity tied to the weight Gram matrix and explains, via the dynamics of \Cref{sec4.sub2:dynamics}, why such linearization emerges from training rather than treating it as a post-hoc finding.

More broadly, theoretical accounts of feature learning mainly fall into three streams: 
(1) \textit{case studies} on structured targets~\cite{abbe2022merged, ba2022high, damian2022neural, dandi2025computational}, whereas our results apply to general neural networks; 
(2) \textit{metric studies} that propose descriptive measures of representation quality~\cite{chou2025feature, cohen2020separability, kornblith2019similarity}, which capture whether feature learning occurred but not how and why; and 
(3) phenomenon-driven studies that explain specific behaviors such as Neural Collapse~\cite{mixon2022neural, papyan2020prevalence} or grokking~\cite{power2022grokking, liu2022towards} in isolation.
Relatedly, the lazy-versus-rich literature, spanning the NTK~\cite{jacot2018neural}, mean-field analyses~\cite{mei2018mean}, and $\mu$P-style parameterizations~\cite{yang2021tuning, dey2025dont}, mainly characterizes whether a network escapes the lazy regime by quantifying the degree of movement in weight space.
Our framework is complementary: instead of measuring how far the weights move, it describes where the learned features move to, namely the progressive linearization toward the target captured by the weight Gram matrix.
We aim to unify these threads by deriving feature dynamics directly from the Gram structure.

Closest to our analysis is the Neural Feature Ansatz \cite{radhakrishnan2024mechanism}, which conjectures $W_l^\top W_l \propto \mathbb{E}[\nabla_{h_l} f \cdot \nabla_{h_l} f^\top]$.
Our VCS captures the Gram update dynamics directly (\Cref{thm:vcs}) without invoking the ansatz, tracking the actual update at $\rho=0.98$ versus $0.71$ for AGOP.
A more comprehensive literature review is in \Cref{a1.sub1:rel_works}.

\paragraph{Target Linearity in Pre-trained Transformer}

We measure TL across every layer of BERT~\citep{devlin2019bert} on four downstream tasks comparing randomly initialized, pre-trained, and fine-tuned models in \Cref{a1.sub5:pretrain}. 
The goal is to check whether the dynamics theorized in \Cref{sec4.sub2:dynamics} holds for complex transformers and track what TL illuminates on the pre-training and fine-tuning regime.
We observe pre-training alone produces representations that are substantially more target linear with downstream labels than random initialization.
Fine-tuning then converts this structure into a near-monotone TL profile that saturates within a single epoch and barely improves thereafter. 

\paragraph{Link to Generalization}

From \Cref{fig:empirical}, TL exhibits an interesting pattern when compared to generalization performance.
In both CIFAR-10 and CIFAR-100, the training accuracy exceeds 0.9, while the test accuracy differs significantly (50\% in CIFAR-10 and 20\% in CIFAR-100).
Meanwhile, the TL values differ substantially: the last-layer TL exceeds 0.5 for CIFAR-10 but only 0.15 for CIFAR-100.
This is particularly interesting because TL is measured on the training set, where both models achieve high accuracy.
To further investigate this phenomenon, we evaluate TL on a random-label dataset~\citep{zhang2017understanding} and a modular arithmetic dataset which exhibit grokking~\citep{power2022grokking} in \Cref{a1.sub6:generalization}.

\begin{ack}
This work was supported by the National Research Foundation of Korea (NRF) grant funded by the Korea government (MSIT) (RS-2025-24873052 and RS-2025-25433902).
We thank the anonymous reviewers for their constructive feedback, which substantially improved this paper.
We also thank Seungsang Oh, Donghwan Kim, Dohyun Kwon, and Chulhee Yun for their thoughtful comments.
\end{ack}

\newpage
\bibliography{cite}


\newpage
\appendix

\section{Related Works \& Limitations}
\label{a1:rel_lim}

\subsection{Comprehensive Related Works}
\label{a1.sub1:rel_works}

As feature learning is one of the distinguishing advantages of neural networks over classical machine learning models (such as kernel machines), an emerging body of work has focused on characterizing both theoretical and empirical aspects of feature learning.  From the theoretical side, the study of feature learning in neural networks has largely been focused on fully connected networks used for multi-index modeling~\cite{abbe2022merged, ba2022high, damian2022neural, dandi2025computational}.  Specifically, multi-index models are functions $f^*: \mathbb{R}^{d} \to \mathbb{R}$ of the form $f^*(x) = g(U^Tx)$ where $U \in \mathbb{R}^{d \times r}$ satisfies $r << d$.  It is known in such settings that neural networks outperform kernel machines by estimating the subspace $U$ through feature learning~(see e.g.,~\cite{damian2022neural} for an example of such an argument).  Select work~\cite{karp2021local} has also theoretically analyzed feature learning in more general neural network architectures such as convolutional networks.  Additional work has studied feature learning in deep fully connected networks, identifying problem settings for which feature learning in 2-hidden layer networks leads to improved sample complexity over feature learning in 1-hidden layer networks~\cite{nichani2023provable}.

As a step toward more broadly characterizing the mechanism of feature learning in neural networks, the work of~\cite{radhakrishnan2024mechanism} posited the Neural Feature Ansatz (NFA) stating that Gram matrices of weights in layer $\ell$ of a neural network are proportional to the Average Gradient Outer Product (AGOP) of the network (where gradients are computed with respect into layer $\ell$).  This work presented empirical evidence for this claim across fully connected networks, convolutional networks, recurrent networks and transformers.  Follow up work~\cite{radhakrishnan2025linear} proved that the claim of~\cite{radhakrishnan2024mechanism} was true in deep linear networks used for low rank matrix completion (a prominent setting for studying feature learning~\cite{gunasekar2018implicit}).  Additional follow up work~\cite{boix-adsera2026fact} showed that the NFA did not hold with exact proportionality for all data generating processes and instead connected Gram matrices with an alternative object (based on gradient of the loss and activations).  Related theoretical work~\cite{parkinson2025relu} also connected the Expected Gradient Outer Product (EGOP) with the implicit bias of fully connected networks with a single nonlinearity.  In particular, this work showed that the representation cost of such networks was bounded by Schatten-norms of the EGOP (implying that networks with minimum representation cost also had low rank EGOPs).  

Lastly, linearization of features is a prominent observation in neural networks that has been studied in several settings~\cite{alain2016understanding, brahma2015deep, mikolov2013efficient}.  Recently, there has been renewed interest in understanding the linearization of features in transformer-based language models, where linear representations of concepts can be used to monitor and steer such models~\cite{beaglehole2026toward, zou2023representation}.

\subsection{Limitations}
\label{a1.sub2:limit}

\paragraph{On First-Order Approximations}
Our analysis of Target Linearity dynamics relies on first-order Taylor expansions, especially the homogeneity-based decompositions (\Cref{lemma:homogeneous}).
This is a deliberate trade-off. 
Existing dynamics analyses~\citep{abbe2022merged, ba2022high, damian2022neural} achieve precise descriptions at the cost of restrictive assumptions on architecture or data. 
We aim for the opposite end of this trade-off: macroscopic dynamics under minimal technical assumptions, capturing leading-order behavior rather than exact trajectories. 

While first-order approximations can be loose for individual terms given the highly nonlinear nature of deep networks, we conjecture that the aggregate quantities we study --- particularly the surrogate $\mathcal{S}(G_{\mathrm{id}})$, which sums over all training samples --- are more robust than any individual term within them, in line with the concentration patterns observed in mean-field analyses~\citep{mei2018mean, sirignano2020mean}. 
We view this approximation-based approach as complementary to exact analyses: trading precision for generality may surface phenomena, such as the sequential linearization of \Cref{sec4.sub2:dynamics}, that restricted exact analyses would not naturally reveal.
The following proposition partially justifies this approach.
Although the proposition assumes exact homogeneity, we empirically find that the aggregate cancellation persists for non-homogeneous GELU networks, with an average error below $10^{-6}$ at every layer.

\begin{proposition}
    \label{prop.first_order}
    If a function $f$ is $m$-positively homogeneous and input data and function values are centered, i.e., $\bar{x} = \frac{1}{N} \sum_i x_i = 0$ and $\sum_i f(x_i)=0$, then sum of all pairs of first-order Taylor approximation errors among each data points is zero.
\end{proposition}

\begin{proof}
    By writing the sum of first-order Taylor approximation errors,
    \begin{align*}        
    \sum_{ij} \left[ f(x_i) - f(x_j) - (x_i - x_j)^\top \nabla_{x_j} f\right] &= \sum_{ij} f(x_i) - \sum_{ij} f(x_j) - \sum_{ij} (x_i - x_j)^\top \nabla_{x_j} f \\
    &= - \sum_{ij} (x_i - x_j)^\top \nabla_{x_j} f \\
    &= - N \bar{x}^\top \left( \sum_{j} \nabla_{x_j} f \right) + N \sum_{j} x_j^\top \nabla_{x_j} f \\
    &= N \sum_j x_j^\top \nabla_{x_j} f \quad \text{as }\bar{x} = \sum_i x_i = 0 \\
    &= mN \sum_j f(x_j) \quad \text{as }f\text{ is homogeneous} \\
    &= 0  \quad \text{as } \sum_i f(x_i) = 0
    \end{align*}
\end{proof}

\paragraph{On Target Linearity and Generalization}
Since TL is measured on the training set, high TL does not by itself imply good generalization.
In our random-label memorization experiment in \Cref{a1.sub6:generalization}, TL keeps increasing during training even when all labels are randomly shuffled ($p = 1$), while the test accuracy remains near chance level.
This shows that a network can linearize memorized, non-generalizable labels, so TL and generalization can diverge.
On the other hand, we observe that the layer-wise TL profile appears more informative about generalization than the final-layer value alone: under random labels, the middle-layer TL exhibits a sudden increase for large $p$, and under grokking, the TL gap between intermediate and last layers closely tracks the generalization error.
However, we have not established this relationship theoretically or validated it beyond these settings, and we leave a principled characterization of how layer-wise TL profiles relate to generalization as future work.

\paragraph{On Technical Assumptions}
Our analysis mainly adopts three assumptions that are standard in the deep learning theory literature: (i) m-positive homogeneity~\citep{ji2020directional, kumar2024early, Lyu2020Gradient}, satisfied exactly by bias-free ReLU networks; and (ii) lower-bounded pre-activation norm, which is automatically satisfied unless $WH = 0$ (a degenerate collapse). 
Additionally, (iii) the recent  Neural Feature Ansatz~\citep{radhakrishnan2024mechanism} is adopted in dynamic analysis.
NFA has been established theoretically in low-rank matrix completion, a widely studied setting for feature learning~\cite{radhakrishnan2025linear}, and has also been validated empirically across a broad range of architectures and datasets.
Additionally, \citet{boix-adsera2026fact} discuss why the NFA usually holds and the degenerate cases under which it does not.
Although its theoretical foundations are not yet fully understood, our analysis relies on extensive empirical validation in the original work.

\section{Experiment Details \& Additional Experiments}
\label{a2:expr}

We use a machine with AMD Ryzen 9 5900X 12-Core Processor CPU with one NVIDIA RTX 3090 GPU.
Each experiment takes less than an hour, and the time can vary depending on the worker's settings.
We attach our code for full reproduction.

\subsection{Virtual Update on the Swiss Roll Dataset}
\label{a2.sub0:swissroll}

\Cref{fig:swissroll} is intended to give intuition for feature linearization before the formal framework is introduced, and we describe its experimental setup here.
We generate 1,000 points from the Swiss roll dataset using \texttt{scikit-learn}, keep the two coordinates that span the spiral, and min-max normalize each coordinate to $[0, 1]$.
The target $y$ is the position of each point along the spiral, also min-max normalized to $[0, 1]$, which is shown by color in the figure.
We train a fully connected network with two hidden layers of width 32 and GELU activations, where the first layer has no bias so that its weight matrix alone determines the Gram matrix $W^\top W$.
The network is trained with the squared loss using full-batch gradient descent with a learning rate of $0.05$, so that each epoch corresponds to a single gradient step.

At each step $t$, we record the input gradient $\nabla_X \mathcal{L}_t$ under the current parameters $\theta_t$, and compute the virtual update by accumulating these gradients, $\tilde{X}_t = X - \gamma \sum_{s \le t} \nabla_X \mathcal{L}_s$.
Since the input gradients are small, we use $\gamma = 700$ for visualization and min-max normalize $\tilde{X}_t$ to $[0, 1]$ in each coordinate before plotting.
Note that the virtual update is only tracked and never fed back into training, so it does not affect the actual optimization trajectory.

The red dashed contours show the density of a Gaussian centered at the mean of $X$ with covariance $W_0^\top W_0 + 10^{-3} I$, where $W_0$ is the first-layer weight matrix at the corresponding epoch and the small ridge term ensures positive definiteness.
We plot the original inputs, the virtually updated inputs, and the contours at epochs 0, 10, 20, 30, 40, and 50.
As training proceeds, the virtual update unrolls the spiral into a near-linear curve ordered by the target, while the contours of $W_0^\top W_0$ reshape accordingly, consistent with \Cref{thm:vcs}.

\subsection{Gradient Whitening Experiment}
\label{a2.sub0:whiten}

We solve the following optimization problem:
\[
\arg_{W^+} \min \|W^+ - W\|_F \quad \text{subject to} \quad (W^+)^\top W^+ = (W - \gamma G)^\top (W - \gamma G),
\]
where $G = \nabla_W \mathcal{L}$.
By writing $W^+ = W - \gamma \Delta$, we can obtain the following problem:
\begin{align*}
    \arg_{\Delta} \min \|\Delta\|_F \quad &\text{subject to} \quad (W - \gamma \Delta)^\top (W - \gamma \Delta) = (W - \gamma G)^\top (W - \gamma G) \\
    & \Leftrightarrow \Delta^\top W + W^\top \Delta \approx G^\top W + W^\top G
\end{align*}
which is approximately equivalent to the original problem:
\begin{align*}
    \arg_{\Delta'} \min \|G - \Delta'\|_F \quad &\text{subject to} \quad \Delta'^\top W + W^\top \Delta' \approx 0 \\
    & \Leftrightarrow (W - \gamma \Delta')^\top (W - \gamma \Delta') \approx W^\top W,
\end{align*}
by letting $\Delta = G - \Delta'$.

Solution for the problem is well-known as an orthogonal procrustes problem:
\[
W^+ = U V^\top \tilde{W^+},
\]
for singular value decomposition $W (\tilde{W^+})^\top = U \Sigma V^\top$ and any $\tilde{W^+}$ satisfying $(W^+)^\top W^+ = (W - \gamma G)^\top (W - \gamma G)$.
By simply setting $\tilde{W^+} = W - \gamma G$, we apply the Newton-Schulz method similar to Muon optimizer~\citep{liu2025muon} to compute $UV^\top$.

\subsection{Sensitivity of Gradient Whitening to Architecture and Hyperparameters}
\label{a2.sub3:whiten_sensitivity}

We further examine whether the gap between standard GD and the whitened update depends on the network size or on the optimization hyperparameters.
We train fully connected networks on CIFAR-10 and CIFAR-100 while varying one factor at a time: the width $\in \{128, 256, 512\}$, the depth $\in \{2, 3, 4\}$, the learning rate $\in \{0.1, 0.05, 0.01, 0.005\}$, and the momentum $\in \{0.8, 0.6, 0.4, 0.2\}$.
\Cref{tab:whiten_arch} reports the results for width and depth, and \Cref{tab:whiten_hparam} reports the results for learning rate and momentum.

\begin{table}[h]
\centering
\caption{Test accuracy (\%) of standard GD (Orig.) and the whitened update (Whit.) for varying width and depth. We report mean $\pm$ standard deviation over 5 runs. The higher mean in each pair is in bold.}
\label{tab:whiten_arch}
\resizebox{\textwidth}{!}{%
\begin{tabular}{llcccccccc}
\toprule
& & \multicolumn{2}{c}{Epoch 5} & \multicolumn{2}{c}{Epoch 10} & \multicolumn{2}{c}{Epoch 15} & \multicolumn{2}{c}{Epoch 20} \\
\cmidrule(lr){3-4} \cmidrule(lr){5-6} \cmidrule(lr){7-8} \cmidrule(lr){9-10}
Dataset & Setting & Orig. & Whit. & Orig. & Whit. & Orig. & Whit. & Orig. & Whit. \\
\midrule
\multirow{6}{*}{CIFAR-10}
& Width 128 & 43.39 $\pm$ 1.69 & \textbf{46.86} $\pm$ 0.36 & 45.45 $\pm$ 1.72 & \textbf{49.72} $\pm$ 0.73 & 46.02 $\pm$ 1.61 & \textbf{50.02} $\pm$ 0.72 & 46.89 $\pm$ 0.49 & \textbf{50.32} $\pm$ 0.80 \\
& Width 256 & 44.41 $\pm$ 0.64 & \textbf{49.62} $\pm$ 0.87 & 47.20 $\pm$ 0.28 & \textbf{50.42} $\pm$ 0.76 & 47.94 $\pm$ 0.63 & \textbf{50.97} $\pm$ 0.70 & 47.29 $\pm$ 1.12 & \textbf{50.93} $\pm$ 0.70 \\
& Width 512 & 45.15 $\pm$ 1.12 & \textbf{49.03} $\pm$ 0.40 & 48.21 $\pm$ 0.88 & \textbf{51.20} $\pm$ 0.71 & 49.77 $\pm$ 0.66 & \textbf{52.17} $\pm$ 0.36 & 49.91 $\pm$ 0.35 & \textbf{52.42} $\pm$ 0.78 \\
\cmidrule(lr){2-10}
& Depth 2 & 44.41 $\pm$ 0.64 & \textbf{49.62} $\pm$ 0.87 & 47.20 $\pm$ 0.28 & \textbf{50.42} $\pm$ 0.76 & 47.94 $\pm$ 0.63 & \textbf{50.97} $\pm$ 0.70 & 47.29 $\pm$ 1.12 & \textbf{50.93} $\pm$ 0.70 \\
& Depth 3 & 45.64 $\pm$ 0.73 & \textbf{47.61} $\pm$ 1.32 & 49.26 $\pm$ 1.05 & \textbf{51.08} $\pm$ 0.70 & 49.57 $\pm$ 0.73 & \textbf{51.94} $\pm$ 0.44 & 50.19 $\pm$ 1.19 & \textbf{51.53} $\pm$ 0.92 \\
& Depth 4 & 45.71 $\pm$ 0.72 & \textbf{46.67} $\pm$ 1.06 & 49.08 $\pm$ 0.64 & \textbf{50.84} $\pm$ 0.49 & 50.77 $\pm$ 1.26 & \textbf{51.32} $\pm$ 0.84 & 51.01 $\pm$ 0.44 & \textbf{52.02} $\pm$ 1.02 \\
\midrule
\multirow{6}{*}{CIFAR-100}
& Width 128 & \textbf{20.58} $\pm$ 0.28 & 17.80 $\pm$ 0.20 & \textbf{24.00} $\pm$ 0.26 & 20.76 $\pm$ 0.21 & \textbf{25.40} $\pm$ 0.25 & 21.82 $\pm$ 0.48 & \textbf{25.81} $\pm$ 0.23 & 22.25 $\pm$ 0.33 \\
& Width 256 & \textbf{19.77} $\pm$ 0.19 & 18.30 $\pm$ 0.24 & \textbf{24.47} $\pm$ 0.34 & 22.27 $\pm$ 0.25 & \textbf{26.56} $\pm$ 0.14 & 24.49 $\pm$ 0.14 & \textbf{27.29} $\pm$ 0.33 & 25.50 $\pm$ 0.26 \\
& Width 512 & \textbf{22.03} $\pm$ 0.13 & 20.29 $\pm$ 0.35 & \textbf{26.84} $\pm$ 0.11 & 24.80 $\pm$ 0.17 & \textbf{28.46} $\pm$ 0.38 & 26.93 $\pm$ 0.30 & \textbf{28.60} $\pm$ 0.22 & 27.47 $\pm$ 0.25 \\
\cmidrule(lr){2-10}
& Depth 2 & \textbf{19.77} $\pm$ 0.19 & 18.30 $\pm$ 0.24 & \textbf{24.47} $\pm$ 0.34 & 22.27 $\pm$ 0.25 & \textbf{26.56} $\pm$ 0.14 & 24.49 $\pm$ 0.14 & \textbf{27.29} $\pm$ 0.33 & 25.50 $\pm$ 0.26 \\
& Depth 3 & \textbf{19.12} $\pm$ 0.22 & 16.73 $\pm$ 0.20 & \textbf{24.52} $\pm$ 0.11 & 22.11 $\pm$ 0.20 & \textbf{26.53} $\pm$ 0.45 & 24.46 $\pm$ 0.14 & \textbf{27.47} $\pm$ 0.16 & 25.48 $\pm$ 0.34 \\
& Depth 4 & \textbf{18.47} $\pm$ 0.28 & 16.48 $\pm$ 0.39 & \textbf{24.52} $\pm$ 0.26 & 22.67 $\pm$ 0.09 & \textbf{26.58} $\pm$ 0.35 & 25.02 $\pm$ 0.36 & \textbf{26.08} $\pm$ 0.40 & 25.21 $\pm$ 0.45 \\
\bottomrule
\end{tabular}%
}
\end{table}

\begin{table}[h]
\centering
\caption{Test accuracy (\%) of standard GD (Orig.) and the whitened update (Whit.) for varying learning rate and momentum. The higher value in each pair is in bold.}
\label{tab:whiten_hparam}
\resizebox{0.85\textwidth}{!}{%
\begin{tabular}{llcccccccc}
\toprule
& & \multicolumn{2}{c}{Epoch 5} & \multicolumn{2}{c}{Epoch 10} & \multicolumn{2}{c}{Epoch 15} & \multicolumn{2}{c}{Epoch 20} \\
\cmidrule(lr){3-4} \cmidrule(lr){5-6} \cmidrule(lr){7-8} \cmidrule(lr){9-10}
Dataset & Setting & Orig. & Whit. & Orig. & Whit. & Orig. & Whit. & Orig. & Whit. \\
\midrule
\multirow{8}{*}{CIFAR-10}
& LR 0.1   & 10.04 & \textbf{44.06} & 10.01 & \textbf{45.22} & 10.01 & \textbf{45.39} & 10.01 & \textbf{44.56} \\
& LR 0.05  & 16.27 & \textbf{46.41} & 13.29 & \textbf{46.95} & 11.13 & \textbf{47.29} & 10.03 & \textbf{46.63} \\
& LR 0.01  & 44.41 & \textbf{49.62} & 47.20 & \textbf{50.42} & 47.94 & \textbf{50.97} & 47.29 & \textbf{50.93} \\
& LR 0.005 & \textbf{48.29} & 48.20 & 50.70 & \textbf{51.13} & 51.49 & \textbf{51.95} & \textbf{52.89} & 52.74 \\
\cmidrule(lr){2-10}
& Momentum 0.8 & \textbf{45.41} & 45.23 & \textbf{51.14} & 50.43 & \textbf{51.88} & 51.50 & \textbf{53.55} & 53.51 \\
& Momentum 0.6 & \textbf{42.13} & 41.62 & \textbf{48.41} & 47.79 & \textbf{48.88} & 48.20 & \textbf{51.70} & 50.91 \\
& Momentum 0.4 & \textbf{39.12} & 38.60 & \textbf{46.13} & 45.33 & \textbf{45.76} & 44.86 & \textbf{49.40} & 48.38 \\
& Momentum 0.2 & \textbf{36.72} & 36.30 & \textbf{43.65} & 42.83 & \textbf{43.18} & 42.05 & \textbf{47.00} & 46.35 \\
\midrule
\multirow{8}{*}{CIFAR-100}
& LR 0.1   & 14.81 & \textbf{20.54} & 15.09 & \textbf{20.11} & 15.37 & \textbf{19.54} & 15.31 & \textbf{18.55} \\
& LR 0.05  & 22.18 & \textbf{23.27} & 21.78 & \textbf{21.92} & 20.94 & \textbf{21.47} & 20.24 & \textbf{20.47} \\
& LR 0.01  & \textbf{22.85} & 21.62 & \textbf{26.75} & 25.05 & \textbf{26.98} & 25.63 & \textbf{26.27} & 24.87 \\
& LR 0.005 & \textbf{19.77} & 18.30 & \textbf{24.47} & 22.27 & \textbf{26.56} & 24.49 & \textbf{27.29} & 25.50 \\
\cmidrule(lr){2-10}
& Momentum 0.8 & \textbf{22.85} & 21.62 & \textbf{26.75} & 25.05 & \textbf{26.98} & 25.63 & \textbf{26.27} & 24.87 \\
& Momentum 0.6 & \textbf{19.68} & 18.32 & \textbf{24.37} & 22.31 & \textbf{26.34} & 24.47 & \textbf{27.19} & 25.43 \\
& Momentum 0.4 & \textbf{17.67} & 16.37 & \textbf{22.13} & 20.00 & \textbf{24.49} & 22.59 & \textbf{26.17} & 24.19 \\
& Momentum 0.2 & \textbf{16.26} & 14.70 & \textbf{20.36} & 18.51 & \textbf{22.99} & 20.83 & \textbf{24.76} & 22.71 \\
\bottomrule
\end{tabular}%
}
\end{table}

\paragraph{Width and depth.}
The gap between standard GD and the whitened update does not grow with network size.
On CIFAR-10, the whitened update outperforms standard GD at every width and depth.
On CIFAR-100, standard GD retains a small advantage, but the gap at epoch 20 shrinks as the width increases (3.56\%p, 1.79\%p, and 1.13\%p for widths 128, 256, and 512) and stays within 2\%p across depths.

\paragraph{Learning rate and momentum.}
The whitened update is notably more stable at high learning rates.
On CIFAR-10, standard GD collapses to chance-level accuracy at learning rates of 0.1 and 0.05, whereas the whitened update maintains 44--47\% accuracy.
On CIFAR-100, the whitened update also outperforms standard GD at these high learning rates, while standard GD leads by about 1--2\%p at lower learning rates.
Across all momentum values, the whitened update stays within about 1\%p of standard GD on CIFAR-10 and about 2\%p on CIFAR-100.

Overall, these results are consistent with \Cref{sec2:gram}: the Gram-modifying component of the gradient recovers most of the performance of standard GD across architectures and hyperparameters, and it is more robust to large learning rates.

\subsection{Gradient Whitening on Convolutional Neural Network}
\label{a2.sub1:whiten_cnn}

We present the gradient whitening results for a Convolutional Neural Network~\citep{krizhevsky2012imagenet}.
We train a model with 3 convolutional layers with 16 channels and 2-by-2 kernels, 2 max pooling layers, followed by a final linear classifier.
We train the model using the SGD optimizer with a learning rate of 0.01 and momentum factor of 0.95 for 20 epochs.
The result in \Cref{a1.fig:cnn_whiten} shows negligible difference between original SGD and Gram-whitened SGD, similar to fully connected neural network case.

\begin{figure}[h]
    \centering
    \includegraphics[width=0.995\linewidth]{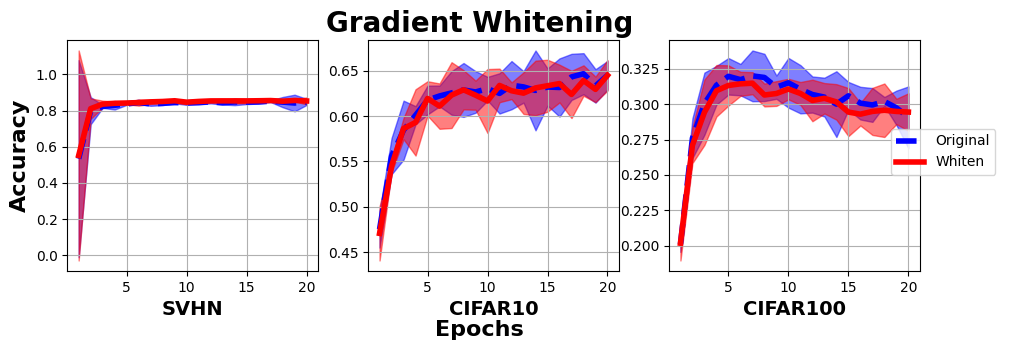}
    \caption{Performance comparison between gradient-whitened updates and standard gradient descent in convolutional neural network}
    \label{a1.fig:cnn_whiten}
\end{figure}

We further evaluate gradient whitening in a more complex training regime by training ResNet-18~\citep{he2016deep} on TinyImageNet, which consists of 100{,}000/10{,}000/10{,}000 train/validation/test images across 200 classes.
We train for 25 epochs using SGD with a learning rate of 0.1 and momentum of 0.9, and report the mean test accuracy over 3 runs in \Cref{tab:whiten_resnet}.
Consistent with the CIFAR-100 results in \Cref{fig:whitening}, standard GD and the whitened update follow similar trends, while the whitened update slightly underperforms.

We conjecture that this gap arises because the Gram matrix $W^\top W$ captures the structure of the input dimension of each layer but not its output dimension, so tasks with many classes, where the output dimension of the last layer is large, are comparatively affected.
To test this, we additionally fine-tune only the last linear layer with standard SGD at the end of every epoch, while training all other layers with the whitened update.
This reduces the gap to standard GD from 2.50\%p to 0.71\%p at epoch 25, and reaches 47.34\% accuracy after only 5 epochs, substantially faster than standard GD.
Although our goal is to show that the Gram matrix captures most of the training dynamics of GD, this result also suggests that Gram-based updates may lead to an efficient training algorithm, which we leave for future work.

\begin{table}[h]
\centering
\caption{Test accuracy (\%) of ResNet-18 on TinyImageNet for standard GD (Orig.), the whitened update (Whit.), and the whitened update with last-layer fine-tuning by standard SGD at the end of each epoch (Whit. + LL FT). We report the mean over 3 runs. The best value at each epoch is in bold.}
\label{tab:whiten_resnet}
\begin{tabular}{lccccc}
\toprule
Method & Epoch 5 & Epoch 10 & Epoch 15 & Epoch 20 & Epoch 25 \\
\midrule
Orig.          & 37.13 & 41.27 & 46.69 & \textbf{49.21} & \textbf{49.23} \\
Whit.          & 36.46 & 38.45 & 46.24 & 46.72 & 46.73 \\
Whit. + LL FT  & \textbf{47.34} & \textbf{46.40} & \textbf{48.32} & 48.53 & 48.52 \\
\bottomrule
\end{tabular}
\end{table}

\subsection{Virtual Covariance vs. AGOP}
\label{a1.sub2:vsagop}

In this section, we compare weight Gram ($W^\top W$) with AGOP ($\tfrac{1}{N} \sum_{i=1}^N \nabla_{h_i} f \, \cdot \nabla_{h_i} f^\top$) and the sum of our proposed VCS ($\sum_{t=1}^\tau \widetilde{\mathrm{cov}}(h^+_{l,t}) - \widetilde{\mathrm{cov}}(h_{l,t})$).
We employ a two-layer neural network with a hidden dimension of 256 and train it on the CIFAR-10 dataset~\citep{krizhevsky2009learning} using SGD with a learning rate of 0.05.
We use the GELU function~\citep{hendrycks2016gaussian} as a $\sigma$.
The model achieves 46\% accuracy on the evaluation dataset.
We then plot the reshaped diagonal components of $W^\top W$, AGOP, and the sum of the VCS in \Cref{a1.fig:vc_vs_agop}.
We also compute the Pearson correlation $\rho$ between them.

\begin{figure}[h]
    \centering
    \includegraphics[width=0.995\linewidth]{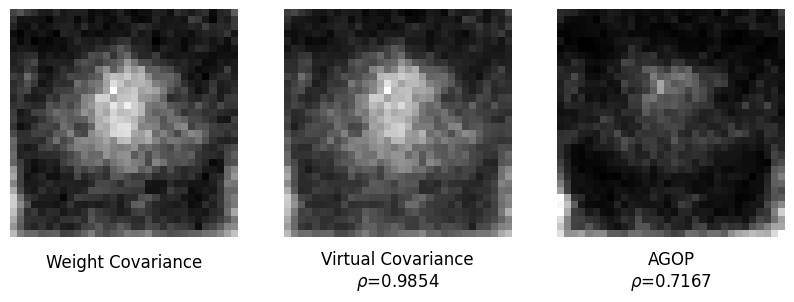}
    \caption{Visualization of the diagonal components of the weight Gram, AGOP, and our proposed Virtual Covariance Shift}
    \label{a1.fig:vc_vs_agop}
\end{figure}

We observe that both VC and AGOP exhibit a high correlation with the weight Gram.
However, VC achieves a correlation of $\rho \approx 0.98$, which is substantially higher than that of AGOP ($\rho \approx 0.71$).
This tendency is also evident in the visualizations: while both methods highlight the centered region, AGOP produces a relatively vague pattern, whereas VC closely resembles the weight Gram.
Through this, we empirically verify that \Cref{thm:vcs} holds, and the $\gamma^2$ factor does not significantly affect the resulting weight Gram.

\subsection{Comparing Lazy vs. Rich Regimes}
\label{a1.sub2:lazy}

We train standard-parameterization MLPs (depth 3, ReLU, no bias) of widths $\{32, 128, 512, 1024\}$ on a synthetic staircase target $y = x_1x_2x_3x_4 + x_1x_2x_3 + x_1x_2 + x_1$~\citep{abbe2022merged} with $n=1000$ Gaussian inputs in $\mathbb{R}^{10}$. 
NTK Initialization~\citep{jacot2018neural} ($\mathrm{std}=1/\sqrt{\mathrm{fan\text{-}in}}$) and learning rate ($0.01$, full-batch SGD, $2000$ epochs) are held fixed across widths, so width alone interpolates between rich and lazy regimes. 
Every $20$ epochs we extract the penultimate features $H$ and record $\|y - y_{\mathrm{OLS}}\|_2$, $\|\hat y - y_{\mathrm{OLS}}\|_2$, and $\|y - \hat y\|_2$, where $y_{\mathrm{OLS}} = H^\top(HH^\top + \epsilon I)^{-1}Hy$.

\begin{figure}[h]
\centering
\includegraphics[width=\linewidth]{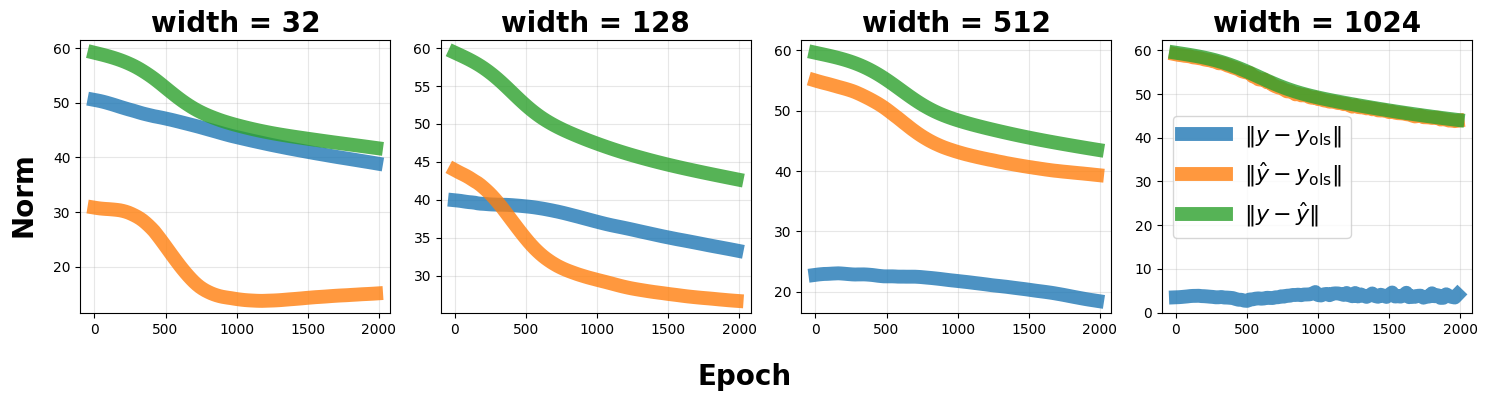}
\caption{Training dynamics of MLPs of increasing width on a staircase target. In the rich regime (widths $32, 128$), $y_{\mathrm{OLS}}$ shifts substantially as features evolve, so the network chases a moving target. As width grows, features freeze: at width $1024$, $\|y - y_{\mathrm{OLS}}\|_2$ is flat near zero.}
\label{a2.fig:compare_dynamics}
\end{figure}

By the triangle inequality, the training loss decomposes as $\|\hat y - y\|_2 \leq \|\hat y - y_{\mathrm{OLS}}\|_2 + \|y_{\mathrm{OLS}} - y\|_2$, isolating two distinct sources of progress: how well the network fits the best linear predictor in its current feature space, and how well that feature space can linearly express $y$.
\Cref{a2.fig:compare_dynamics} shows that in the rich regime (widths $32, 128$), both terms decrease substantially: $\hat y$ chases $y_{\mathrm{OLS}}$ while $y_{\mathrm{OLS}}$ itself shifts toward $y$ as features reorganize.
As width grows, this gap flattens; at width $1024$, $\|y - y_{\mathrm{OLS}}\|_2$ stays near zero throughout training while $\|\hat y - y_{\mathrm{OLS}}\|_2$ alone drives the decrease in $\|y - \hat y\|_2$, recovering the lazy / NTK behavior in which learning reduces to kernel regression in a fixed feature space.
Consequently, the rich regime achieves a lower final $\|\hat y - y\|_2$ through successful feature learning, whereas the lazy regime is bottlenecked by the static expressivity of its initial features.
This directly confirms the moving-target picture predicted in \Cref{sec4.sub2:dynamics}.

\subsection{Target Linearity with ADAM Optimizer}
\label{a1.sub3:adam}

We present the surrogate and TL values with ADAM~\citep{kingma2014adam} optimizer in \Cref{a1.fig:adam}.
Though our analysis mainly focused on the gradient descent optimizer, the results show a consistent pattern for various optimizers.

\begin{figure}[h]
    \centering
    \includegraphics[width=0.495\linewidth]{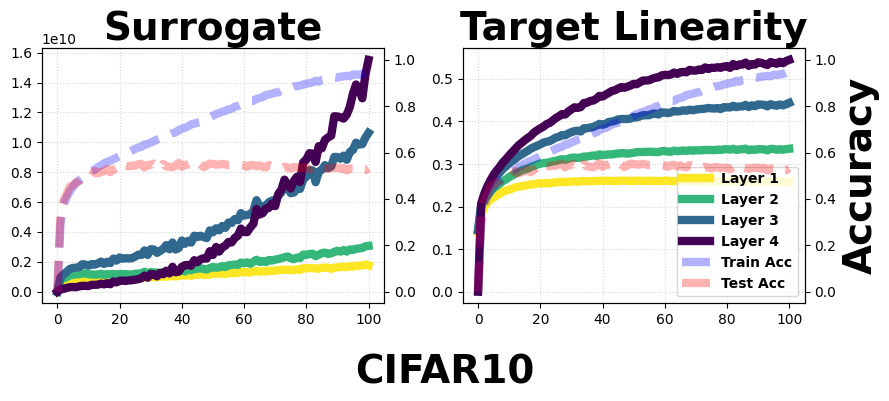}
    \includegraphics[width=0.495\linewidth]{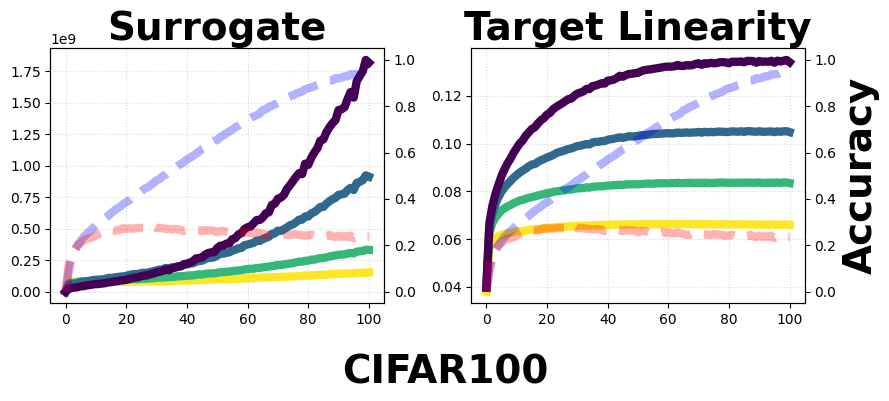}
    \caption{Surrogate and Target Linearity when trained with ADAM optimizer}
    \label{a1.fig:adam}
\end{figure}

\subsection{Target Linearity in Generative Model}
\label{a1.sub4:generative}

We empirically measure how TL evolves in a self-supervised VAE~\citep{kingma2013auto} trained on an image generation task.
Assuming Gaussianity of the latent vector, the ($\beta$-)VAE is trained with the following loss function:
\[
\mathcal{L} = \underbrace{- E_{q(z|x)} [\log p (x|z)]}_{\text{reconstruction error}} + \underbrace{\beta D_{KL} (q(z|x) || p(z))}_{\text{regularization error}},
\]
where the regularization term controlled by $\beta$ constrains the latent space to follow a Gaussian distribution.

We train the VAE on two image datasets: MNIST and CIFAR-10.
We use a 2-layer fully-connected neural network with hidden dimension 512 for both the encoder and decoder, and set the latent dimension to 32.
The VAE is trained with the ADAM optimizer~\citep{kingma2014adam} with a learning rate of 0.0001 for 10 epochs for MNIST and 30 epochs for CIFAR-10.
During training, we measure the TL between latent vectors $Z$ and the prediction target $X$.
We visualize the TL along with the validation set reconstruction error for varying $\beta$ from 0.05 to 10.0.
The results are shown in \Cref{a1.fig:tl_vae}.

\begin{figure}[h]
    \centering
    \includegraphics[width=0.98\linewidth]{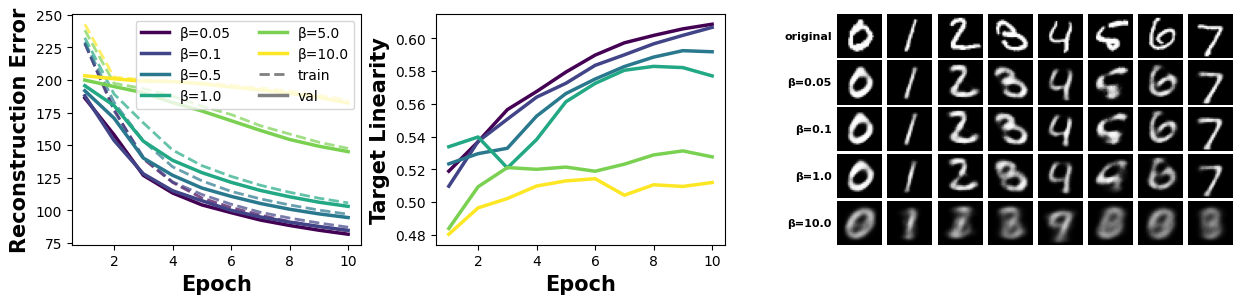}
    \includegraphics[width=0.98\linewidth]{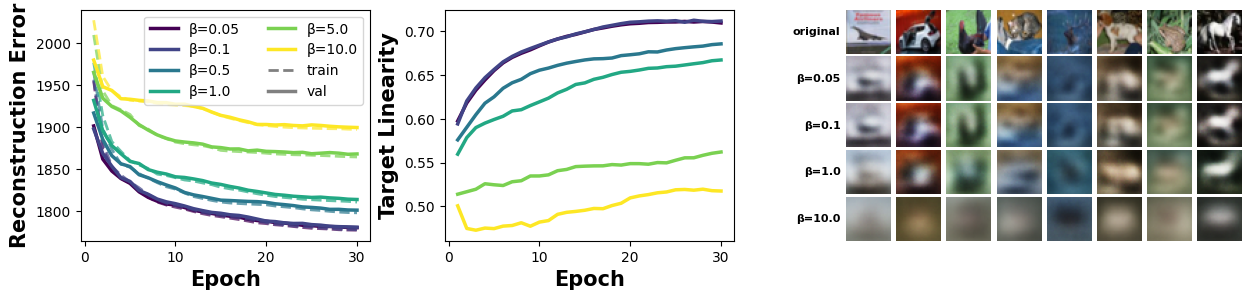}
    \caption{Reconstruction error, Target Linearity, and decoded output for VAE trained on MNIST (upper) and CIFAR-10 (lower)}
    \label{a1.fig:tl_vae}
\end{figure}

The results demonstrate several key observations about the relationship between the regularization parameter $\beta$, reconstruction quality, and TL in VAE training. 

First, we observe a clear trade-off between reconstruction error and TL across different $\beta$ values. 
For both MNIST and CIFAR-10 datasets, lower $\beta$ values (0.05, 0.1) achieve significantly lower reconstruction errors but also exhibit higher TL. 
Conversely, higher $\beta$ values (5.0, 10.0) result in increased reconstruction error while maintaining lower TL. 
This trade-off is visually evident in the generated images: lower $\beta$ values produce sharper, more detailed reconstructions that closely resemble the original images, whereas higher $\beta$ values yield increasingly blurred and abstract representations.

Second, the evolution of TL during training shows distinct patterns depending on $\beta$. 
For low $\beta$ values (0.05, 0.1), TL steadily increases throughout training, suggesting that the latent representations become progressively more linearly related to the input as the model prioritizes reconstruction accuracy. 
In contrast, moderate to high $\beta$ values (1.0, 5.0, 10.0) maintain relatively stable or slightly increasing TL, indicating that stronger regularization constrains the latent space structure.

Finally, the generated samples reveal the practical implications of the $\beta$ parameter. 
For MNIST, even with $\beta=10.0$, the digits remain somewhat recognizable, albeit heavily smoothed. However, for CIFAR-10, high $\beta$ values ($\beta=10.0$) produce nearly uniform gray images with minimal structure, indicating that excessive regularization can completely collapse the model's ability to capture meaningful image features in complex domains.

\subsection{Target Linearity in Pre-trained Transformers}
\label{a1.sub5:pretrain}

We fine-tune BERT~\citep{devlin2019bert} on four tasks spanning two task families: sequence classification (SST-2 and MNLI datasets \citep{wang2018glue}) and token classification (CoNLL-2003 POS tagging and NER datasets \citep{tjong2003conll}). 
All models are trained for 3 epochs with AdamW~\citep{kingma2014adam} at learning rate 2e-5.
To compute target linearity at each stage, we sample 2,000 training examples and reduce hidden states to 256 dimensions via PCA.
We report TL for five model states per task: randomly initialized (architecture only), pre-trained (no fine-tuning), and the three fine-tuned epoch checkpoints.
The results are in \Cref{a1.fig:bert}.

\begin{figure}[h]
    \centering
    \includegraphics[width=0.98\linewidth]{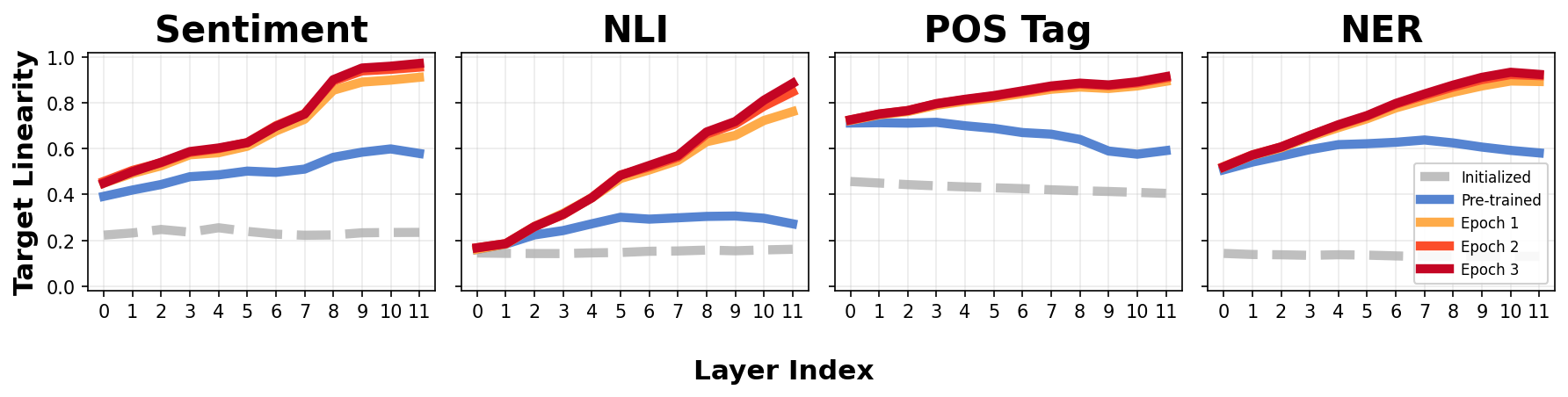}
    \caption{Layerwise Target Linearity across four tasks. TL at each layer of BERT for randomly initialized (gray dashed), pre-trained (blue), and fine-tuned (orange and red, epochs 1–3) models. Measured on SST-2 and MNLI, CONLL 2003 POS and NER datasets.}
    \label{a1.fig:bert}
\end{figure}

Across all four tasks, the pre-trained model's TL curve sits well above the randomly initialized baseline at every layer, despite the pre-training objective (MLM + NSP) never seeing any of these labels. 
This is the cleanest evidence in our experiments that pre-training does not merely produce a useful initialization for optimization but actually learns representations that are relatively linearly separable for downstream tasks. 
The size of the gap, however, depends sharply on how aligned the task is with the pre-training objective: it is largest for POS and NER (which depend on local lexical and syntactic regularities that MLM directly rewards), moderate for sentiment, and smallest for NLI --- where the pre-trained TL barely separates from the random baseline. 
NLI requires cross-sentence reasoning that token-level reconstruction has little incentive to encode, so the linearly-extractable signal is correspondingly weaker.

Pre-training and fine-tuning produce qualitatively different layerwise profiles. 
The pre-trained curves recapitulate a pattern reported in earlier probing work~\citep{tenney2019bert}: TL for semantic tasks (sentiment) climbs monotonically toward the top of the network, while TL for syntactic tasks (POS, NER) peaks in the early-to-middle layers and slightly decays afterward. 
This is consistent with the view that BERT progressively assembles a classical NLP pipeline, with surface and syntactic information concentrated low and semantic information high. 
Fine-tuning, however, erases this division: after a single epoch, all four tasks exhibit the same near-monotone increasing TL profile, with the pre-trained curve serving as a lower envelope.
Even POS --- for which the pre-trained model's optimal layer is layer 3 --- is reorganized so that the final layer carries the most target linear information. 
This implies the layerwise specialization seen during pre-training appears to be a property of the pre-training objective rather than an intrinsic property of the architecture.

\subsection{Target Linearity and Generalization}
\label{a1.sub6:generalization}

\paragraph{Random Label Memorization}

From \Cref{fig:empirical}, TL exhibits an interesting pattern when compared to generalization performance.
In both CIFAR-10 and CIFAR-100, the training accuracy exceeds 0.9, while the test accuracy differs significantly.
The model achieves over 50\% test accuracy on CIFAR-10 but only over 20\% on CIFAR-100, indicating a much larger generalization gap for CIFAR-100.
The TL values also differ substantially: the last-layer TL exceeds 0.5 for CIFAR-10 but only 0.15 for CIFAR-100.
This is particularly interesting because TL is measured on the training set, where both models achieve high accuracy.
These results suggest a potential relationship between generalization and TL, although we do not provide a theoretical analysis.

To further investigate this phenomenon, we measure TL on randomly labeled datasets.
It is well known that neural networks can memorize randomly labeled data, resulting in high generalization error~\citep{zhang2017understanding}.
We report TL values and train/test accuracy for various values of $p \in [0, 1]$, where $p$ denotes the fraction of data with randomly assigned labels.
When $p=0$, the labels are identical to the original dataset, whereas when $p=1$, all labels are randomly shuffled.
We train a 3-layer neural network with a hidden dimension of 256 on CIFAR-10 for 200 epochs to allow sufficient memorization.
The model is optimized using stochastic gradient descent with a learning rate of 0.005.
The results are presented in \Cref{a1.fig:random_label}.

\begin{figure}[h]
    \centering
    \includegraphics[width=0.99\linewidth]{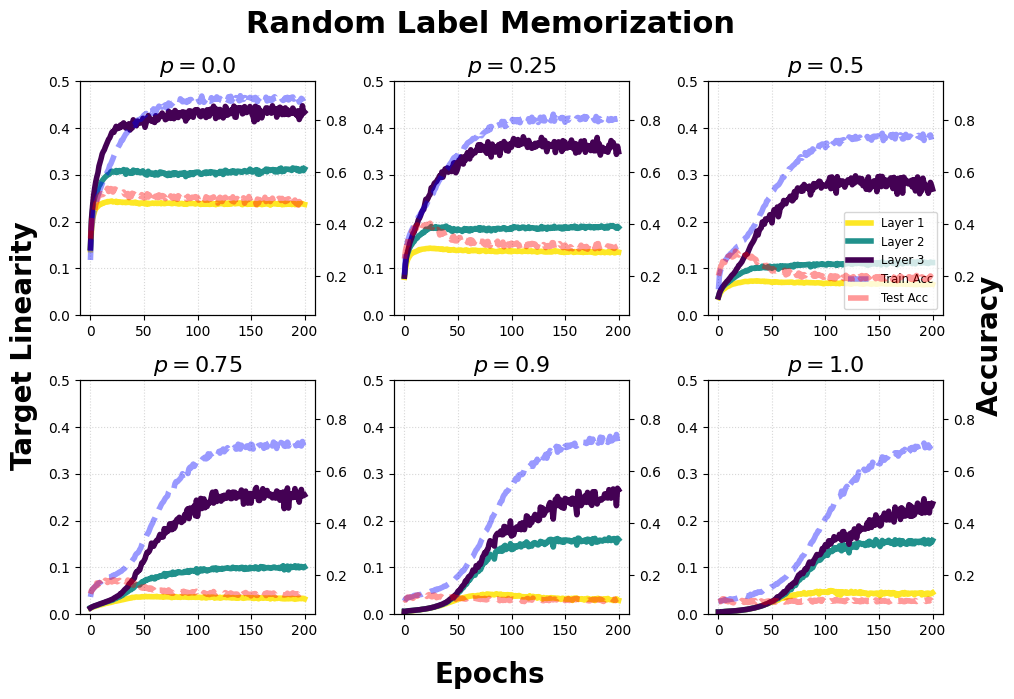}
    \caption{Target Linearity measured on randomly labeled CIFAR 10}
    \label{a1.fig:random_label}
\end{figure}

\Cref{a1.fig:random_label} shows that increasing the random label probability $p$ leads to both a larger generalization error and a decrease in last-layer TL.
As $p$ increases, the overall training accuracy decreases up to $p=0.5$, while it converges to approximately $0.7$ for $p \geq 0.75$.
In contrast, the test accuracy decreases more rapidly, from $0.5$ ($p=0.0$) to $0.3$ ($p=0.25$) and $0.2$ ($p=0.5$), resulting in an increasing generalization gap.
Although the final training accuracy remains around $0.7$ and the test accuracy around $0.1$ for $p \geq 0.75$, higher values of $p$ require more training time to reach high training accuracy.
As a result, during the middle stages of training, the generalization gap is larger for higher $p$.

Interestingly, TL also decreases substantially: the last-layer TL drops from $0.45$ ($p=0.0$) to $0.3$ ($p=0.5$), while the training accuracy decreases from $0.85$ to $0.75$.
TL in other layers also decreases over this range.
For $p \geq 0.75$, the last-layer TL at the end of training remains around $0.25$.
However, similar to the slower increase in training accuracy, the growth of TL also slows for higher $p$.
Overall, the alignment between decreasing TL and an increasing generalization gap provides further evidence that TL may serve as a measure of generalization.

More interestingly, the TL of the middle layers exhibits a sudden increase for $p \geq 0.75$, unlike the cases with $p \leq 0.75$.
This may indicate the need for additional layers to fully memorize highly randomized labels.
For $p \leq 0.75$, the model can achieve sufficiently low loss using modestly processed intermediate representations exploiting the last layer, resulting in rapid early increases.
However, for $p \geq 0.75$, the model explores the parameter space more extensively, showing slow early increases followed by sudden changes across multiple layers.
This observation suggests that layer-wise TL patterns may reveal characteristics related to data complexity.

\paragraph{Grokking}

Another interesting generalization-related phenomenon is grokking~\citep{power2022grokking}.
Grokking refers to a phenomenon in which a neural network first memorizes the training data with poor generalization, and only after prolonged training suddenly transitions to a solution that generalizes well.
This behavior was observed in algorithmic tasks such as modular arithmetic, where models trained with standard optimization initially achieve near-perfect training accuracy while test accuracy remains low, followed by an abrupt improvement in generalization after many additional training steps.
The transition is often associated with implicit regularization effects of optimization, such as weight decay, which gradually favor simpler or more structured solutions over memorizing ones.

Here, we examine how TL evolves under various grokking settings.
We primarily vary the weight decay parameter $\lambda$, which is known to control the timing of grokking.
We use a modular addition dataset with a prime modulus of 61.
A 3-layer neural network with a hidden dimension of 256 is trained using the Adam optimizer~\citep{kingma2014adam} with a learning rate of 0.001 for 500 epochs.
By varying $\lambda \in (0, 1)$, we analyze how TL changes as grokking emerges.
The results are presented in \Cref{a1.fig:grokking}.

\begin{figure}[h]
    \centering
    \includegraphics[width=0.98\linewidth]{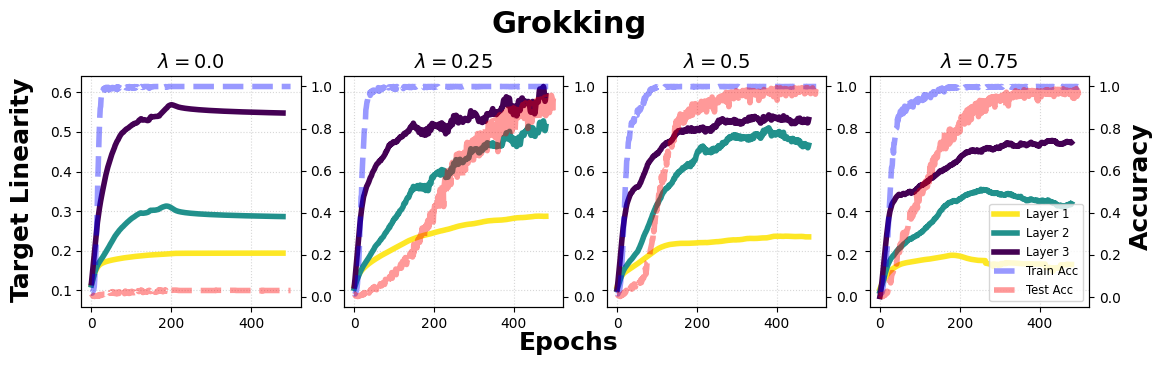}
    \caption{Target Linearity dynamics under Grokking setting}
    \label{a1.fig:grokking}
\end{figure}

\Cref{a1.fig:grokking} shows diverse patterns of overfitting and grokking across different values of the weight decay parameter $\lambda$.
One notable observation is that, for $\lambda \geq 0.25$, TL continues to increase even after memorization, i.e., when the training accuracy is approximately $1.0$.
TL increases until grokking emerges and then saturates in the final stage of training.
These results suggest a potential relationship between TL dynamics and generalization.

\begin{figure}[h]
    \centering
    \includegraphics[width=0.98\linewidth]{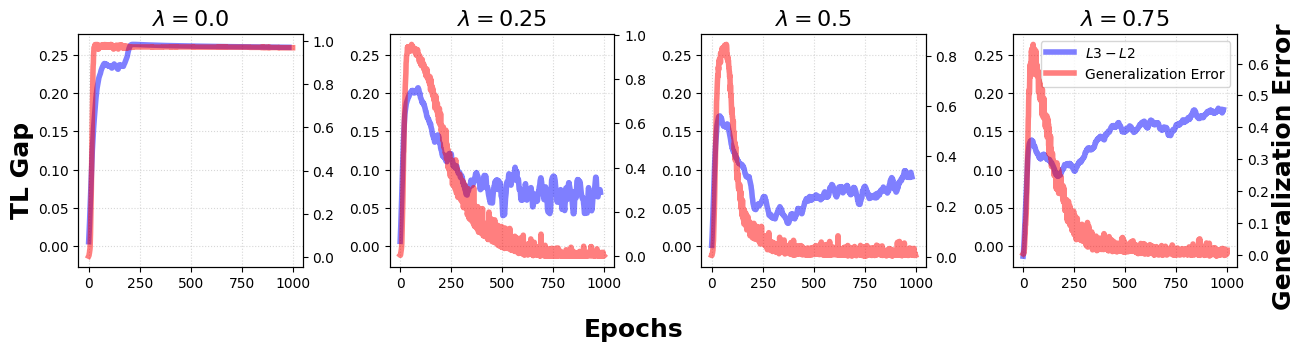}
    \caption{Comparing Target Linearity gap and generalization error}
    \label{a1.fig:grokgap}
\end{figure}

An interesting observation is that, as grokking emerges, the TL of the second layer (Layer 2) increases sharply.
As a result, the gap between the second and last layers (Layer 3) narrows, although it widens again after generalization in the $\lambda = 0.75$ case.
For a more direct comparison, we plot the TL gap between the second and last layers alongside the generalization error in \Cref{a1.fig:grokgap}.
Although the TL gap increases after generalization, particularly for $\lambda = 0.75$, it exhibits a similar decreasing pattern at the onset of grokking.
These results suggest another potential use of layer-wise TL patterns as a measure of generalization.

\section{Proofs}
\label{a2:proof}

\subsection{Proof of \Cref{thm:vcs} in Convolutional Neural Network}
\label{a2.sub0:vcs_cnn}

Without loss of generality, we study a single convolutional layer with nonlinear $\sigma$ and a linear readout,
\[
f(x)\;=\;a^\top\,\sigma(W * x),
\]
where \(x\in\mathbb{R}^{B\times C_{\mathrm{in}}\times H\times W}\), \(W\in\mathbb{R}^{C_{\mathrm{out}}\times C_{\mathrm{in}}\times k_h\times k_w}\), and \(a\in\mathbb{R}^{C_{\mathrm{out}}}\) is a per-channel linear readout (applied at each spatial site).
Let
\[
D := C_{\mathrm{in}}\,k_h\,k_w,\qquad N := B\,H'\,W'.
\]
We use the standard im2col operator \(\mathcal U\) that stacks receptive-field patches into columns:
\[
h \;:=\; \mathcal U(x)\in\mathbb{R}^{D\times N},
\quad
\widetilde W \in \mathbb{R}^{C_{\mathrm{out}}\times D}\ \text{(flattened kernels)}.
\]
Define pre-activations in patch space:
\[
z \;=\; \widetilde W\,h \in \mathbb{R}^{C_{\mathrm{out}}\times N}.
\]

By the chain rule,
\[
\nabla_{\widetilde W}\mathcal L \;=\; \nabla_z \mathcal{L}\,h^\top,\qquad
\nabla_h \mathcal L \;=\; \widetilde W^\top \nabla_z \mathcal{L}.
\]
So we obtain FL equation (\Cref{eq:fle})
\[
\widetilde W^\top\,\nabla_{\widetilde W}\mathcal L  \;=\; \nabla_h \mathcal L \,\cdot h^\top,
\]
which leads to the CNN version of \Cref{thm:vcs}.

\subsection{Proof for \Cref{prop:vc_alignment}}
\label{a2.sub1:vc_alignment}

\begin{statement}
    Assume $\sigma$ is $L$-Lipschitz and $\|h_{l-1}\|_2 = 1$, and let $W^+_{l-1}$ denote the layer $l{-}1$ weight matrix after one gradient descent step on the single data point producing $h_{l-1}$, with all other layers held fixed.
    Then $|\sigma(W^+_{l-1} h_{l-1}) - \sigma(W_{l-1} h_{l-1})| \leq L^2 |h_l^+ - h_l|$ element-wise.
    If $\sigma$ is additionally monotonically increasing, then $\mathrm{sgn}(\sigma(W^+_{l-1} h_{l-1}) - \sigma(W_{l-1} h_{l-1})) = \mathrm{sgn}(h_l^+ - h_l)$; the sign equality also holds for non-decreasing $\sigma$ such as ReLU under the convention $\mathrm{sgn}(0) \in \{-1, +1\}$.
\end{statement}

\begin{proof}
    Define the pre-activation hidden state $z_l = W_{l-1}h_{l-1}$ and its actual update $z_l^+ = W_{l-1}^+h_{l-1}$.
    Then we can observe 
    \begin{align*}
        z_l^+ &=  W_{l-1}^+ h_{l-1} \\
        &= W_{l-1} h_{l-1} - \gamma (\nabla_{W_{l-1}} \mathcal{L}) h_{l-1} \\
        &= z_l - \gamma \nabla_{z_l} \mathcal{L} \langle h_{l-1}, h_{l-1}\rangle \quad \text{as } \nabla_{W_{l-1}} \mathcal{L} = \nabla_{z_l} \mathcal{L} \cdot h_{l-1}^\top \\
        &= z_l - \gamma \nabla_{z_l} \mathcal{L} \quad \text{as }\lVert h_{l-1} \rVert =1.
    \end{align*}
    By the $L$-Lipschitz property, 
    \begin{align*}
    |\sigma(W^+_{l-1} h_{l-1}) - \sigma(W_{l-1} h_{l-1})| &\leq L| W^+_{l-1} h_{l-1} - W_{l-1} h_{l-1} | \\
    &= L| z_{l}^+ - z_{l} | \\
    &= L |  \gamma \nabla_{z_{l}} \mathcal{L} | \\
    &= L | \gamma \sigma'(z_l) \odot \nabla_{h_{l}} \mathcal{L} | \quad \text{as } \nabla_{z_{l}} \mathcal{L}=\sigma'(z_l) \odot \nabla_{h_{l}} \mathcal{L}\\
    &\leq L^2 | \gamma \nabla_{h_{l}}  \mathcal{L} | \quad \text{as } \sigma'(z_l)_i \le L, \forall i \\
    &= L^2 | h_l^+ - h_l |
    \end{align*}

    For the sign case, as $\sigma$ is a non-decreasing function,
    \begin{align*}
        \text{sgn}( \sigma (W_{l-1}^+ h_{l-1}) - \sigma (W_{l-1} h_{l-1}) ) &= \text{sgn}( \sigma (z_l^+) - \sigma (z_l) ) \\
        &= \text{sgn}( z_l^+ - z_l ) \quad \text{by non-decreasing property}\\
        &= \text{sgn}( - \nabla_{z_l} \mathcal{L} ) \\
        &= \text{sgn} ( - \sigma'(z_l) \odot \nabla_{h_{l}} \mathcal{L} ) \\
        &= \text{sgn} ( - \nabla_{h_{l}} \mathcal{L} ) \quad \text{by non-decreasing property} \\
        &= \text{sgn}(h_l^+ - h_l).
    \end{align*}
\end{proof}

\subsection{Proof for Proposition~\ref{prop:all_layer}}
\label{a2.sub2:proof_all_layer}

\begin{statement}
    Consider a single input $x$ and an $L$-layer network with $z_l = W_{l-1} h_{l-1}$ and $h_l = \sigma(z_l)$ for $l = 1, \dots, L$, where $h_0 = x$ and $\sigma$ is differentiable.
    Suppose all weights evolve under gradient flow, $\dot{W}_{l} = -\nabla_{W_{l}} \mathcal{L}$ for $l = 0, \dots, L-1$.
    Then for every layer $l$,
    \begin{equation*}
        \frac{d h_l}{dt} = -M_l \nabla_{h_l} \mathcal{L}, \qquad M_l = \sum_{l'=1}^{l} \|h_{l'-1}\|_2^2 \, J_{l \leftarrow l'} J_{l \leftarrow l'}^\top \succeq 0,
    \end{equation*}
    where $J_{l \leftarrow l'} = \partial h_l / \partial z_{l'}$.
    Consequently, $\langle \frac{d h_l}{dt}, -\nabla_{h_l} \mathcal{L} \rangle \ge 0$.
\end{statement}

\begin{proof}
    For $1 \le l' \le l \le L$, define the diagonal activation-derivative matrices and the Jacobians
    \begin{equation*}
        D_l := \mathrm{diag}(\sigma'(z_l)), \qquad J_{l \leftarrow l'} := \frac{\partial h_l}{\partial z_{l'}} = D_l W_{l-1} D_{l-1} W_{l-2} \cdots W_{l'} D_{l'},
    \end{equation*}
    with the convention $J_{l \leftarrow l} = D_l$.
    Since $h_l$ depends only on $W_0, \dots, W_{l-1}$, the chain rule gives
    \begin{equation*}
        \frac{d h_l}{dt} = \sum_{l'=1}^{l} \partial_{W_{l'-1}} h_l [\dot{W}_{l'-1}],
    \end{equation*}
    where $\partial_{W} h_l [A]$ denotes the directional derivative of $h_l$ with respect to $W$ in direction $A$.
    Since $W_{l'-1}$ enters $h_l$ only through $z_{l'} = W_{l'-1} h_{l'-1}$, and $h_{l'-1}$ does not depend on $W_{l'-1}$,
    \begin{equation*}
        \partial_{W_{l'-1}} h_l [A] = \frac{\partial h_l}{\partial z_{l'}} (A h_{l'-1}) = J_{l \leftarrow l'} A h_{l'-1}.
    \end{equation*}
    Inserting $A = \dot{W}_{l'-1} = -\nabla_{W_{l'-1}} \mathcal{L} = -\nabla_{z_{l'}} \mathcal{L} \, h_{l'-1}^\top$ yields
    \begin{equation*}
        \frac{d h_l}{dt} = -\sum_{l'=1}^{l} J_{l \leftarrow l'} \nabla_{z_{l'}} \mathcal{L} \, h_{l'-1}^\top h_{l'-1} = -\sum_{l'=1}^{l} \|h_{l'-1}\|_2^2 \, J_{l \leftarrow l'} \nabla_{z_{l'}} \mathcal{L}.
    \end{equation*}
    Since every path from $z_{l'}$ to the loss passes through $h_l$ for $l' \le l$, we have $\nabla_{z_{l'}} \mathcal{L} = J_{l \leftarrow l'}^\top \nabla_{h_l} \mathcal{L}$, and therefore
    \begin{equation*}
        \frac{d h_l}{dt} = -\Big( \sum_{l'=1}^{l} \|h_{l'-1}\|_2^2 \, J_{l \leftarrow l'} J_{l \leftarrow l'}^\top \Big) \nabla_{h_l} \mathcal{L} = -M_l \nabla_{h_l} \mathcal{L}.
    \end{equation*}
    Each summand of $M_l$ is PSD, so $M_l \succeq 0$ and $\langle \frac{d h_l}{dt}, -\nabla_{h_l} \mathcal{L} \rangle = \nabla_{h_l} \mathcal{L}^\top M_l \nabla_{h_l} \mathcal{L} \ge 0$.
\end{proof}

\subsection{Proof for \Cref{thm:surrogate}}
\label{a2.sub2:surrogate}

\begin{statement}
    Let $e_0 = \|Y - \bar{Y}\|_2^2$ and $e_1 = \|Y\|^2_2$, and assume $\|W\|_{\mathrm{op}} \leq c_0$, $\|H\|_F \leq c_1$, and $N > d$.
    Then
    \[
        \mathcal{T}_\lambda(H, Y) \geq 1 - C \cdot \mathcal{S}(G_{\mathrm{id}})^{-1},
    \]
    for some fixed constant $C$.
\end{statement}

If $W,H,G$ are $l$-layer weight, hidden state, and Gram, define $W',H',G'$ as $l-1$-layer quantities.
We first show the relationship between $\mathcal{S}(G_{\text{id}})$ and $\mathcal{S}(G'_{\sigma})$.

\begin{lemma}
    With a regularity condition on weight matrix $\| W \|_{op} < c$, $\mathcal{S}(G'_{\sigma}) > c^{-2} \cdot \mathcal{S}(G_{\text{id}})$.
    \label{lemma:gid_gsigma}
\end{lemma}

\begin{proof}[Proof for \Cref{lemma:gid_gsigma}]
    We can show
    \begin{align*}
        \mathcal{S}(G_{\text{id}}) = Y^\top H^\top W^\top W H Y < c^2 \cdot Y^\top H^\top H Y  = c^2 \cdot Y^\top \sigma(W' H')^\top \sigma(W' H') Y = c^2 \cdot \mathcal{S}(G'_{\sigma}),
    \end{align*}
    where the inequality holds by the definition of the matrix operator norm.
\end{proof}

\Cref{lemma:gid_gsigma} shows that increasing the $l$-th layer linear $\mathcal{S}(G_{\text{id}})$ increases the lower bound of $l-1$-th layer non-linear $\mathcal{S}(G'_{\sigma})$.
This arguments holds for all $l=2, ..., L$ except the first $1$-st layer where the hidden state $H$ is equal to the fixed input data $X$.

Since we showed the relationship between two surrogates, we now show how  $\mathcal{S}(G_{\sigma})$ is related to our target quantity  $\mathcal{E}_{\lambda}(G_{\sigma})$.

\begin{lemma}
    For a symmetric positive (semi-) definite matrix $G \in \mathbb{R}^{N \times N}$, vector $Y \in \mathbb{R}^{N}$, and a positive constant $\lambda>0$, the following holds:
    \begin{align*}
        Y^\top (\lambda \cdot I + G)^{-2} Y \leq \kappa \cdot \|Y\|_2^4 \cdot \left[ Y^\top (\lambda \cdot I + G) Y \right]^{-1}
    \end{align*}
    where $\kappa = \frac{\left( 2\lambda + \lambda_{\max} + \lambda_{\min} \right)^2 }{4 (\lambda +  \lambda_{\min})^{3}}$, and $\lambda_{\max}, \lambda_{\min}$ are maximum and minimum eigenvalue of $G$.
    \label{lemma:error_and_s_sigma}
\end{lemma}

\begin{proof}[Proof for \Cref{lemma:error_and_s_sigma}]
    By spectral decomposition, $G = Q \Lambda Q^\top$, $Y^\top \left[ \lambda I + G\right] Y = \sum_{i} (\lambda + \lambda_i) \cdot \langle q_i, Y \rangle^2 $.
    Similarly, $Y^\top \left[ \lambda I + G\right]^{-a} Y = \sum_{i} (\lambda + \lambda_i)^{-a} \cdot \langle q_i, Y \rangle^2$ for integer $a$.
    
    By the Kantorovich inequality,
    \begin{align*}
        \left( Y^\top \left[ \lambda \cdot I + G \right] Y\right) \left( Y^\top \left[ \lambda \cdot I + G \right]^{-1} Y\right) \leq \|Y\|_2^4 \frac{\left( 2\lambda +  \lambda_{\max} + \lambda_{\min} \right)^2 }{4 (\lambda +\lambda_{\max}) \cdot (\lambda + \lambda_{\min})}
    \end{align*}
    and the bounds
    \begin{align*}
        & (\lambda + \lambda_{\max})^{-1} \cdot \sum_{i} \langle q_i, Y \rangle^2 \leq Y^\top \left[ \lambda \cdot I + G \right]^{-1} Y \quad \text{and} \\
        & Y^\top \left[ \lambda \cdot I + G \right]^{-2} Y \leq (\lambda + \lambda_{\min})^{-2} \cdot \sum_{i} \langle q_i, Y \rangle^2 \\
        & \Rightarrow Y^\top \left[ \lambda \cdot I + G \right]^{-2} Y \leq (\lambda + \lambda_{\max})^{-1} \cdot \sum_{i}\langle q_i, Y \rangle^2  \cdot \frac{\lambda + \lambda_{\max}}{(\lambda + \lambda_{\min})^{2}} \\
         &\quad\quad\quad\quad\quad\quad \leq Y^\top \left[ \lambda \cdot I + G \right]^{-1} Y \cdot \frac{\lambda + \lambda_{\max}}{(\lambda + \lambda_{\min})^{2}}.
    \end{align*}
    So we can conclude
    \begin{align*}
        Y^\top \left[ \lambda \cdot I + G \right]^{-2} Y &\leq \kappa \cdot \|Y\|_2^4  \cdot \left( Y^\top \left[ \lambda \cdot I + G \right] Y\right)^{-1}, \\
        \text{where }\kappa &= \frac{\left( 2\lambda + \lambda_{\max} + \lambda_{\min} \right)^2 }{4 (\lambda+\lambda_{\max}) \cdot(\lambda +  \lambda_{\min})} \cdot \frac{\lambda + \lambda_{\max}}{(\lambda + \lambda_{\min})^{2}} \\
        & = \frac{\left( 2\lambda + \lambda_{\max} + \lambda_{\min} \right)^2 }{4 (\lambda +  \lambda_{\min})^{3}}.
    \end{align*}
\end{proof}

Since $\mathcal{S}(G_{\sigma})=Y^\top G_{\sigma} Y \propto Y^\top (\lambda \cdot I + G_{\sigma}) Y $ and $\mathcal{E}_{\lambda}(G_{\sigma})=Y^\top (\lambda I + G_{\sigma})^{-2} Y$, we can conclude that increasing $\mathcal{S}(G_{\sigma})$ decreases the upper bound of $\mathcal{E}_{\lambda}(G_{\sigma})$.

Combining these lemmas, we obtain the resulting proof.

\begin{proof}
    Since $\mathcal{T}_{\lambda}(H, Y) = 1 - \frac{\|Y-\hat{Y}_{\lambda}\|^2_2}{\|Y-\bar{Y}\|^2_2} = 1 - \frac{\lambda^2 \mathcal{E}_{\lambda}(G'_{\sigma})}{e_0}$, where
    \begin{align*}
        \mathcal{E}_{\lambda}(G'_{\sigma}) &\leq \kappa_0 \cdot \|Y\|_2^4 \cdot \left[ Y^\top (\lambda \cdot I + G'_{\sigma}) Y \right]^{-1} \quad \text{by \Cref{lemma:error_and_s_sigma}} \\
        &\leq \kappa_0 e_1^2 \cdot \left[ Y^\top G'_{\sigma} Y \right]^{-1} \quad \text{since $\lambda Y^\top Y \geq 0$}\\
        &= \kappa_0 e_1^2 \cdot \mathcal{S}(G'_{\sigma})^{-1} \\
        & \leq \kappa_0 e_1^2 c_0^{2} \cdot \mathcal{S}(G_{\text{id}})^{-1} \quad \text{by \Cref{lemma:gid_gsigma}}, 
    \end{align*}
    where 
    \begin{align*}
        \kappa_0 &= \frac{\left( 2\lambda + \lambda_{\max} + \lambda_{\min} \right)^2 }{4 (\lambda +  \lambda_{\min})^{3}} \quad \text{where $\lambda_{min}, \lambda_{max}$ are min/max eigenvalues of $G'_\sigma$}\\
        &= \frac{\left( 2\lambda + \lambda_{\max} \right)^2 }{4 \lambda^{3}} \quad \text{since $\lambda_{min}=0$ as $N > d$ } \\
        & \leq \frac{\left( 2\lambda + c_1^2 \right)^2 }{4 \lambda^{3}}
    \end{align*}
    where the last line holds as $\lambda_{max} = \|H\|_{op}^2 \leq \|H\|_{F}^2 \leq c_1^2$.
    We obtain the result by setting $C = e_0^{-1} e_1^2 c_0^2 \cdot \frac{(2\lambda + c_1^2)^2}{4\lambda}$.
\end{proof}

\subsection{Proof for \Cref{thm:training}}
\label{a2.sub3:training}

\begin{statement}
For any loss function $\mathcal{L}$, let $G_{\mathrm{id}} = H^\top W^\top W H$ and $G_{\mathrm{id}}^+ = H^\top (W^+)^\top W^+ H$, where $W^+ = W - \gamma \nabla_W \mathcal{L}$.
If $f$ is $1$-positively homogeneous in $h$, the following holds:
\[
\mathcal{S}(G_{\mathrm{id}}^+) - \mathcal{S}(G_{\mathrm{id}}) \approx 2 \gamma \, (f^\top y) \cdot (y^\top K g),
\]
where $K_{ij} = h_i^\top h_j$, $f_i = f(h_i)$ are the predictions on the training set, and $g = -\nabla_f \mathcal{L}$.
\end{statement}

We first state and prove a lemma:

\begin{lemma}
    If a function $f$ is $m$-positively homogeneous in $h$, then $h_i^\top \nabla_{h_k} f \approx f(h_i) + (m-1) f(h_k)$.
    \label{lemma:homogeneous}
\end{lemma}

\begin{proof}[Proof of \Cref{lemma:homogeneous}]
    By Taylor expansion, we can write $f(h_i) \approx f(h_k) + (h_i - h_k)^\top \nabla_{h_k}f$.
    Since $f$ is the $m$-positively homogeneous, $h^\top \nabla_hf = m f(h)$ holds.
    So we obtain
    \begin{align*}
        f(h_i) - f(h_k) &\approx h_i^\top \nabla_{h_k}f - h_k^\top \nabla_{h_k}f \\
        &= h_i^\top \nabla_{h_k}f - m f(h_k). \\
        \Rightarrow h_i^\top \nabla_{h_k}f  &\approx f(h_i) + (m - 1) f(h_k)
    \end{align*}
\end{proof}

Now we prove the main statement.

\begin{proof}
    Recall the Feature Learning Equation (\Cref{eq:fle}) and its Gram form (\Cref{eq:gram}),
    \[
        (W^+)^\top W^+ - W^\top W \approx - \gamma \left[ \sum_{i} h_i \cdot \nabla_{h_i} \mathcal{L}^\top + \nabla_{h_i} \mathcal{L} \cdot h_i^\top \right].
    \]
    By writing $\mathcal{S}(G_{\text{id}})= \sum_{ij} [\mathcal{S}(G_{\text{id}})]_{ij}$ where $[\mathcal{S}(G_{\text{id}})]_{ij} = y_i y_j h_i^\top W^\top W h_j$, we have
    \begin{align*}
        [\mathcal{S}(G^+_{\mathrm{id}})]_{ij} - [\mathcal{S}(G_{\text{id}})]_{ij} &= y_i y_j h_i^\top \left((W^+)^\top W^+ - W^\top W \right) h_j \\
        &\approx - \gamma y_i y_j h_i^\top \left[ \sum_{k} h_k \cdot \nabla_{h_k} \mathcal{L}^\top + \nabla_{h_k} \mathcal{L} \cdot h_k^\top \right] h_j\\
        &= - \gamma \sum_{k} y_i y_j h_i^\top h_k \cdot \nabla_{h_k} \mathcal{L}^\top h_j - y_i y_j h_i^\top\nabla_{h_k} \mathcal{L} \cdot h_k^\top h_j.
    \end{align*}
    Note $\nabla_h \mathcal{L} = \nabla_h f \cdot \nabla_f \mathcal{L} = - \nabla_h f \cdot g$, and $h_i^\top \nabla_{h_k} f \approx f(h_i) + (m-1) f(h_k) = f(h_i)$ by \Cref{lemma:homogeneous} and $m=1$.
    So we can write
    \begin{align*}
        \mathcal{S}(G^+_{\mathrm{id}}) - \mathcal{S}(G_{\text{id}}) &= \sum_{ij} [\mathcal{S}(G^+_{\mathrm{id}})]_{ij} - [\mathcal{S}(G_{\text{id}})]_{ij} \\
        & \approx - \gamma \left[ \sum_{ijk} y_i y_j h_i^\top h_k \cdot \nabla_{h_k} \mathcal{L}^\top h_j + y_i y_j h_i^\top\nabla_{h_k} \mathcal{L} \cdot h_k^\top h_j \right]\\
        &= - \gamma \left[  \sum_{ijk} 2 y_i y_j h_i^\top\nabla_{h_k} \mathcal{L} \cdot h_k^\top h_j \right] \quad \text{with symmetry}\\
        &= \gamma \sum_{ijk} 2 y_i y_j h_i^\top\nabla_{h_k} f \cdot g_k \cdot h_k^\top h_j \\
        &\approx \gamma \sum_{ijk} 2 y_i y_j g_k f(h_i) \cdot h_k^\top h_j \\
        &= 2 \gamma (f^\top y) \cdot (y^\top K g) \quad \text{where $K_{ij} = h_i^\top h_j$ and $f_i = f(h_i)$}.
    \end{align*}
\end{proof}

\subsection{Proof for \Cref{thm:depth}}
\label{a2.sub4:depth}

\begin{statement}
    Let $f_l: h_l \mapsto f(x)$ denote the subnetwork mapping the $l$-th hidden state to the final output, and assume the activation $\sigma$ is $m$-positively homogeneous.
    Assume the network interpolates the training data ($f(x_i) \approx y_i$, equivalently $f_l(h_{l,i}) \approx y_i$ for all $i$ and $l$) and $\exists C_0 > 0 \text{ s.t. } \|W_l H_l\|^2_F \geq C_0, \forall l$.
    Under the Neural Feature Ansatz $W_l^\top W_l \approx \tfrac{C_l}{N} \sum_i \nabla_{h_{l,i}} f_l \, \nabla_{h_{l,i}} f_l^\top$ for layer-dependent constants $C_l > 0$, the surrogate satisfies
    \[
    \mathcal{S}(G^l_{\mathrm{id}}) \begin{cases}
    \gtrsim C_0 \|Y\|_2^2 \, \Phi^{-1}_l, & \text{if } m > 1 \text{ and } \sum_i y_i = 0, \\
    \gtrsim C_0 \|Y\|_2^4 \, \Psi^{-1}_l, & \text{if } m = 1 \text{ and } f_l \text{ is piecewise linear,}
    \end{cases}
    \]
    where $\Phi^{-1}_l$ and $\Psi^{-1}_l$ are monotonically increasing in $l$.
\end{statement}

For the first case ($m>1$), we use \Cref{lemma:homogeneous}, while for piecewise linear case, we use the following lemma.

\begin{lemma}
    \label{lemma:pwl}
    For a ReLU network $f$ with bounded inputs $\|a\|_2, \|b\|_2 < B$, suppose that the gradient difference across any pair of adjacent linear regions is bounded by $\delta$, i.e.,
    \[
        \|\nabla f(x) - \nabla f(y)\|_2 < \delta,
    \]
    for any $x$ and $y$ lying in two adjacent linear regions.
    Let $M$ denote the number of linear-region boundaries crossed by the segment $[a, b]$.
    Then
    \[
        b^\top \nabla f(a) = f(b) + \epsilon,
        \quad \text{where} \quad
        |\epsilon| < 2MB\delta.
    \]
\end{lemma}

\begin{proof}[Proof of \Cref{lemma:pwl}]
    Recall the ReLU network $f$ is a 1-positively homogeneous and piece-wise linear function.
    By the fundamental theorem of calculus, for a piecewise linear function $f$ and points $a, b$,
    \begin{align*}
    f(b) - f(a) &= \int_0^1 \nabla f(\gamma(t))^\top (b-a) \, dt \\
    &= \sum_{j=0}^{M} (t_{j+1} - t_j) \cdot v_j^\top (b - a),
    \end{align*}
    where $\gamma(t) = a + t(b-a)$, the segment $[a,b]$ crosses $M$ boundaries between linear regions at parameter values $0 = t_0 < t_1 < \cdots < t_M < t_{M+1} = 1$, and $v_j = \nabla f(\gamma(t))$ for $t \in (t_j, t_{j+1})$ denotes the (constant) gradient on the $j$-th sub-interval. 
    In particular, $v_0 = \nabla f(a)$ and $v_M = \nabla f(b)$.
    
    Subtracting the leading term $v_0^\top(b-a) = \nabla f(a)^\top(b-a)$ from both sides yields the first-order Taylor expansion with explicit remainder:
    \begin{align*}
    f(b) &= f(a) + \nabla f(a)^\top (b - a) - \epsilon, \\
    \epsilon &= \sum_{j=1}^{M} (t_{j+1} - t_j) (v_j - v_0)^\top (b - a).
    \end{align*}
    By the $\delta$ bound of $\nabla f$ for two adjacent linear regions, we have
    \[
    \|v_j - v_0\|_2 = \| \sum_{k=1}^j v_{k} - v_{k-1} \|_2 \leq j \delta.
    \]
    As all input points are normalized, $\|b - a \|_2 < 2B$,
    \begin{align*}
    |\epsilon| &\leq \|b - a\|_2 \cdot \sum_{j=1}^{M} (t_{j+1} - t_j) \|v_j - v_0\| _2\\
    & \leq 2B \delta \sum_{j=1}^{M} (t_{j+1} - t_j) j \\
    &= 2B \delta \sum_{j=1}^{M} (t_{j+1} - t_j) \sum_{k=1}^j 1 \\
    &= 2B \delta \sum_{j=1}^{M} \sum_{k=1}^M (t_{j+1} - t_j) \mathbf{1}_{\{k \leq j\}} \\
    &= 2B \delta \sum_{k=1}^{M} \sum_{j=k}^{M} (t_{j+1} - t_j) \\
    &= 2B \delta \sum_{k=1}^{M} (1 - t_k) \leq 2 MB \delta
    \end{align*}
    Since $f(a) = a^\top\nabla f(a)$ holds by the $1$-homogeneity, the result holds.
\end{proof}

From \Cref{lemma:homogeneous}, we showed that via Taylor expansion, $b^\top \nabla f(a)$ can be approximated by $f(b)$ for $1$-homogeneous $f$.
In \Cref{lemma:pwl}, we characterize the approximation error, which is analogous to the norm of the Hessian.
However, the Hessian does not exist for piecewise-linear functions, as the second-order derivative is zero within linear regions and undefined at the boundaries.
We therefore measure a generalized notion of the Hessian in terms of the number of boundaries between linear regions, and use it to bound $\epsilon$.
\Cref{lemma:pwl} intuitively shows the error of the first order Taylor expansion can be bounded by the function complexity measured by the number of jumps ($M$) and the degree of the jump ($\delta$).

Now we prove the main theorem.

\begin{proof}
    As we focus on the fixed $l$-th layer, we omit the $l$ superscripts.
    We can write $\mathcal{S}(G^{l}_{\mathrm{id}}) = \sum_{ij} \mathcal{S}(G^{l}_{\mathrm{id}})_{ij}$ where $\mathcal{S}(G^{l}_{\mathrm{id}})_{ij} = y_iy_j h_i^\top W^\top W h_j$.
    By applying the Neural Feature Ansatz,
    \begin{align*}
        \mathcal{S}(G^{l}_{\mathrm{id}})_{ij} &=C_l \cdot y_iy_j h_i^\top \left[ \frac{1}{N} \sum_{k} \nabla_{h_k} f \cdot \nabla_{h_k} f^\top \right] h_j \\
        &= \frac{C_l}{N} \sum_{k} y_iy_j h_i^\top \nabla_{h_k} f \cdot \nabla_{h_k} f^\top h_j.
    \end{align*}
    Note that $f$ is $l'$-homogeneous, where $l' = m^{L-l}$ is the degree of homogeneity for the remaining layers.
    With \Cref{lemma:homogeneous}, we can rewrite 
    \begin{align*}
        h_i^\top \nabla_{h_k} f \cdot \nabla_{h_k} f^\top  h_j &\approx \left[ f(h_i) + (l' - 1) f(h_k) \right] \cdot \left[ f(h_j) + (l' - 1) f(h_k) \right] \\
        &\approx \left[ y_i + (l' - 1) y_k \right] \cdot \left[ y_j + (l' - 1) y_k \right] \quad \text{(interpolation assumption)}\\
        &= y_i y_j + (l' - 1) y_k y_i + (l' - 1)y_k y_j + (l' - 1)^2 y_k^2.
    \end{align*}
    So
    \begin{align*}
        \sum_{k} y_iy_j h_i^\top \nabla_{h_k} f \cdot \nabla_{h_k} f^\top h_j &\approx \sum_{k} y_i^2 y_j^2 + (l' - 1) y_k y_i^2 y_j + (l' - 1)y_k y_i y_j^2 + (l' - 1)^2 y_k^2 y_i y_j \quad \text{and} \\
         \mathcal{S}(G_{\mathrm{id}}^{l}) &\approx \frac{C_l}{N} \sum_{ijk} y_i^2 y_j^2 + (l' - 1) y_k y_i^2 y_j + (l' - 1)y_k y_i y_j^2 + (l' - 1)^2 y_k^2 y_i y_j \\
         &= C_l \| Y \|_2^4 + 2 (l' - 1) (\sum_i y_i)^2 \| Y \|_2^2 + (l' - 1)^2 (\sum_i y_i)^2 \| Y \|_2^2
    \end{align*}
    \textbf{Case 1 ($m > 1$ and $\sum_i y_i = 0$)} 
    As $\sum_i y_i = 0$, the latter terms in $\mathcal{S}(G_{\mathrm{id}}^{l})$ is $0$, i.e. $\mathcal{S}(G_{\mathrm{id}}^{l}) = C_l \| Y \|_2^4$.
    Note for the pre-activation norm lower bound,
    \begin{align*}
        \| WH \|_F^2 = \text{tr} \left[H^\top W^\top W H \right] &= \frac{C_l}{N} \sum_{ik} h_i^\top \nabla_{h_k} f \cdot \nabla_{h_k} f^\top  h_i \\
        &\approx \frac{C_l}{N} \sum_{ik} \left[ f(h_i) + (l' - 1) f(h_k) \right]^2 \quad \text{with \Cref{lemma:homogeneous}}\\
        &\approx \frac{C_l}{N} \sum_{ik} y_i^2 + \left[ (l' - 1) y_k \right]^2 + 2 (l' - 1) y_i y_k \\
        &= C_l \left( \left( l'-1 \right)^2 + 1 \right) \cdot \| Y \|_2^2 \gtrsim C_0 \quad \text{by assumption.} \\
        \Rightarrow C_l \| Y \|_2^2 & \gtrsim C_0 \left( \left( l'-1 \right)^2 + 1 \right)^{-1} = C_0 \left( \left( m^{L-l} - 1 \right)^2 + 1 \right)^{-1}
    \end{align*}
    So we obtain
    \begin{align*}
        \mathcal{S}(G_{\mathrm{id}}^{l}) &\approx C_l \| Y \|_2^4 \gtrsim C_0 \| Y \|_2^2 \left( \left( m^{L-l} - 1 \right)^2 + 1 \right)^{-1} := C_0 \|Y\|_2^2 \, \Phi^{-1}_l,
    \end{align*}
    where $\Phi_l = \left( \left( m^{L-l} - 1 \right)^2 + 1 \right)$ is monotonically decreasing in $l$, as $C_0, L$, and $\| Y \|_2$ are fixed for all layers, when $m > 1$.
    
    \textbf{Case 2 ($m = 1$, piecewise linear)} 
    As $m=1$, $l' - 1 = m^{L-l} - 1 = 0$, so the latter term in $\mathcal{S}(G_{\mathrm{id}}^{l})$ is $0$, i.e. $\mathcal{S}(G_{\mathrm{id}}^{l}) = C_l \| Y \|_2^4$.
    Define $M_l$ be the number of linear regions in the input space of $l$-th layer.
    Define global constants $B$ and $\delta$ which bounds the norm of inputs and gradient difference between adjacent linear regions accordance with \Cref{lemma:pwl}.
    Then the normalizing constant,
    \begin{align*}
        \| WH \|_F^2 = \text{tr} \left[H^\top W^\top W H \right] &= \frac{C_l}{N} \sum_{ik} h_i^\top \nabla_{h_k} f \cdot \nabla_{h_k} f^\top  h_i \\
        &\approx \frac{C_l}{N} \sum_{ik} \left[ f(h_i) + \epsilon_{ik} \right]^2 \quad \text{with \Cref{lemma:pwl}}\\
        &\approx \frac{C_l}{N} \sum_{ik} \left[ y_i + \epsilon_{ik} \right]^2  \gtrsim C_0 \quad \text{by interpolation assumption.} \\
    \end{align*}
    This implies
    \begin{align*}
        C_l &\gtrsim N C_0 \left[ \sum_{ik} \left[ y_i + \epsilon_{ik} \right]^2 \right]^{-1} \\
        & \geq N C_0 \left[ \left[ \sum_{ik}  | y_i + \epsilon_{ik} | \right]^2 \right]^{-1} \\
        & \geq N C_0 \left[ \left[ \sum_{ik}  | y_i |+ |\epsilon_{ik} | \right]^2 \right]^{-1} = N C_0 \left( N \| Y \|_1 + \sum_{ik} |\epsilon_{ik}| \right)^{-2} \\
    \end{align*}
    
    So we obtain
    \begin{align*}
        \mathcal{S}(G_{\mathrm{id}}) \approx C_l \| Y \|_2^4 &\gtrsim N C_0 \| Y \|_2^4 \left( N \| Y \|_1 + \sum_{ik} |\epsilon_{ik}| \right)^{-2} \\
        &\geq N C_0 \| Y \|_2^4 (N \| Y \|_1 + 2N^2 M_l B\delta)^{-2} \quad \text{by \Cref{lemma:pwl}.} \\
        &= N^{-1}C_0 \| Y \|_2^4 (\| Y \|_1 + 2N M_l B\delta)^{-2} := C_0 \| Y \|_2^4 \Psi^{-1}_l \\
    \end{align*}
    where $\Psi_l = N (\| Y \|_1 + 2N M_l B\delta)^{2}$ is monotonically decreasing in $l$
    as the number of linear region in input space decreases with $l$~\citep{montufar2014number}.
\end{proof}

\subsection{Proof for \Cref{prop:nc_extreme}}
\label{a2.sub5:nc_extreme}

\begin{statement}
Consider $f(x) = W^\top h(x) + b$ with hidden states centered ($\sum_i h(x_i) = 0$), $W$ and $b$ trained under squared loss on the centered one-hot label matrix, and $C \leq D$.
The surrogate $\mathcal{S}_{\mathrm{cls}}(G_{\mathrm{id}})$ attains the maximum value allowed by a norm constraint $\|G_{\mathrm{id}}\|_F \leq c$ for some $c > 0$ under NC1 and NC2.
\end{statement}

To prove the proposition, we need the following lemma.

\begin{lemma}
    For a matrix $A \in \mathbb{R}^{m \times n}$ with $AA^\top \neq 0$ and a symmetric matrix $G \in \mathbb{R}^{m \times m}$ with $\|G\|_F \le c$, $\mathrm{tr}(A^\top G A)$ attains its maximum if and only if $G = \frac{c}{\|AA^\top\|_F} AA^\top$.
    \label{lemma:surrogate_bound}
\end{lemma}

\begin{proof}[Proof of \Cref{lemma:surrogate_bound}]
    Note that $\mathrm{tr}(A^\top G A) = \mathrm{tr}(G A A^\top) = \langle G, AA^\top \rangle_F \le \|G\|_F \cdot \|AA^\top\|_F \le c \|AA^\top\|_F$, where equality holds if and only if $G$ is a nonnegative multiple of $AA^\top$ with $\|G\|_F = c$, i.e., $G = \frac{c}{\|AA^\top\|_F} AA^\top$.
\end{proof}

Now we prove the proposition.

\begin{proof}
    Since the hidden states are normalized, we can write the total, between-class, and within-class variation matrices as $\Sigma_T = \mathrm{Average}_{i} h_i \cdot h_i^\top$, $\Sigma_W=\mathrm{Average}_{c,i} (h_{i,c} - \mu_c) \cdot (h_{i,c} - \mu_c)^\top$, and $\Sigma_B= \mathrm{Average}_{c} \mu_c \cdot \mu_c^\top$, where $h_{i,c}$ is the $i$-th hidden state belonging to the $c$-th class, and $\mu_c = \mathrm{Average}_{i} h_{i,c}$ is the class mean vector.
    Let $M = \left[\mu_c\right]_{c=1}^C$ be the class mean matrix.
    Note that $\Sigma_T = \Sigma_B + \Sigma_W$, and (NC1) implies $\Sigma_W = 0$, so $\Sigma_T=\Sigma_B$.
    
    Following Proposition 1 in \citet{papyan2020prevalence}, the weight matrix satisfies:
    \[
    W = \frac{1}{C} M^\top \Sigma_T^+ = \frac{1}{C} M^\top \Sigma_B^+ = M^\top (MM^\top)^+= M^+,
    \]
    where $+$ denotes the Moore-Penrose pseudoinverse.
    The first equality holds by \citet{webb1990optimised}, and the second equality holds by the (NC1) condition.
    The third equality holds by the definition of $\Sigma_B$, and the last equality holds by properties of the pseudoinverse.

    Since the number of classes $C$ is smaller than the hidden dimension $D$ and the hidden states form the simplex ETF (NC2), $M^+ M = I - \frac{1}{C} \mathbf{1} \cdot \mathbf{1}^\top$.
    This implies $M^+ h_{i,c} = M^+ \mu_{c}$ is a centered one-hot vector with $1 - \frac{1}{C}$ in the $c$-th element and $- \frac{1}{C}$ the others, so that $M^+ H = Y^\top$.
    Thus, the Gram matrix is $G_{\mathrm{id}} = H^\top W^\top W H = Y Y^\top$ and the surrogate is $\mathcal{S}_{\text{cls}}(G_{\mathrm{id}}) = \mathrm{tr} \left[Y^\top Y Y^\top Y\right]$.
    \Cref{lemma:surrogate_bound} shows that $\mathcal{S}_{\text{cls}}(G_{\mathrm{id}})$ is maximized when $G_{\mathrm{id}} = \frac{\| Y Y^\top \|_F}{\| Y Y^\top \|_F} Y Y^\top = Y Y^\top$, which is achieved by setting $c = \| Y Y^\top \|_F$.
\end{proof}

\subsection{Proof for \Cref{prop:vae_linearity}}
\label{a2.sub6:vae_linearity}

\begin{statement}
    Let $Z = [z_1, \ldots, z_N]^\top$ be latent vectors and $X = [x_1, \ldots, x_N]^\top$ the corresponding inputs.
    Let $g(z) = z^\top (Z^\top Z)^{-1} Z^\top X$ be the ordinary least-squares estimator, and assume $g \in \mathcal{F}$.
    Then for any $\alpha \in [0, 1]$,
    \[
    \min_{f \in \mathcal{F}} \tfrac{1}{N^2} \sum_{i,j} \|f(\alpha z_i + (1-\alpha) z_j) - (\alpha x_i + (1-\alpha) x_j)\|_2 \;\leq\; \tfrac{1}{N} \sum_i \|g(z_i) - x_i\|_2.
    \]
\end{statement}

\begin{proof}
        Using the Triangle inequality, 
    \begin{align*}
    & \|f(\alpha z_i + (1-\alpha) z_j) - (\alpha x_i + (1 - \alpha) x_j )\|_2 \\
    & = \|f(\alpha z_i + (1-\alpha) z_j) - g(\alpha z_i + (1-\alpha) z_j) + g(\alpha z_i + (1-\alpha) z_j) -  (\alpha x_i + (1 - \alpha) x_j )\|_2 \\
    & \leq \|f(\alpha z_i + (1-\alpha) z_j) - g(\alpha z_i + (1-\alpha) z_j)\|_2 + \|g(\alpha z_i + (1-\alpha) z_j) -  (\alpha x_i + (1 - \alpha) x_j )\|_2.
    \end{align*}
    So we obtain,
    \begin{align*}
    \min_{f \in \mathcal{F}}  \sum_{ij} & \|f(\alpha z_i + (1-\alpha) z_j) - (\alpha x_i + (1 - \alpha) x_j )\|_2 \\
    & \leq \min_{f \in \mathcal{F}} \sum_{ij} \|f(\alpha z_i + (1-\alpha) z_j) - g(\alpha z_i + (1-\alpha) z_j)\|_2  \\
    & \quad \quad \quad \quad + \sum_{ij} \|g(\alpha z_i + (1-\alpha) z_j) -  (\alpha x_i + (1 - \alpha) x_j )\|_2
    \end{align*}
    We can show the first term is 0 as $g \in \mathcal{F}$.
    For the second term, as $g$ is a linear function, we can show 
    \[
    g(\alpha z_i + (1-\alpha) z_j) = \alpha g(z_i) + (1-\alpha) g(z_j).
    \]
    Then the error term becomes
    \begin{align*}
    \sum_{ij} &\|\alpha g(z_i) + (1-\alpha) g(z_j) -  (\alpha x_i + (1 - \alpha) x_j )\|_2 \\
    & \leq \sum_{ij} \alpha \| g(z_i) -  x_i\|_2 +  (1-\alpha) \|g(z_j) - x_j\|_2 \\
    &=  \alpha \sum_{ij} \| g(z_i) -  x_i\|_2 +  (1-\alpha) \sum_{ij} \|g(z_j) - x_j\|_2 ] \\
    &= \alpha N \sum_{i} \|g(z_i) - x_i\|_2 + (1 - \alpha) N \sum_{i} \|g(z_i) - x_i\|_2 \\
    &= N \cdot \sum_{i} \|g(z_i) - x_i\|_2,
    \end{align*}
    and the result follows by dividing by $N^2$.
\end{proof}


\end{document}